\pdfoutput=1

\RequirePackage[l2tabu, orthodox]{nag}
\documentclass[12pt]{article}
\PassOptionsToPackage{numbers, compress, sort}{natbib}
\usepackage[T1]{fontenc}
\usepackage{microtype}

\usepackage{amsmath}
\usepackage{amssymb}
\usepackage{amsthm}
\usepackage{mathtools}
\usepackage{thmtools}
\usepackage{mathrsfs}

\usepackage{libertine}      

\usepackage{titlesec}

\usepackage{hyperref}

\titleformat*{\section}{\Large\bfseries}

\usepackage[usenames,dvipsnames]{xcolor}
\definecolor{shadecolor}{gray}{0.9}
\usepackage{soul}

\usepackage[english]{babel}
\usepackage[parfill]{parskip}
\usepackage{afterpage}
\usepackage{framed}
\usepackage{nicefrac}

{\endMakeFramed}

\newcounter{parcount}

\usepackage{graphicx}
\usepackage[labelfont=bf]{caption}
\usepackage[format=hang]{subcaption}
\usepackage{wrapfig}

\usepackage{booktabs}
\usepackage{multirow}
\usepackage{multicol}
\usepackage{longtable}
\usepackage{booktabs}       
\usepackage{array}
\newcolumntype{C}[1]{>{\centering\arraybackslash}m{#1}}

\usepackage[sort]{natbib}
\usepackage{bibunits}

\usepackage{algorithm} 
\usepackage{algorithmic}
\newcommand{\INPUT}{\item[\textbf{Input:}]}

\usepackage{hyperref}
\hypersetup{colorlinks,linktoc=all}
\hypersetup{citecolor=Violet}
\hypersetup{linkcolor=black}
\hypersetup{urlcolor=MidnightBlue}

\usepackage[nameinlink]{cleveref}

\usepackage[acronym,nowarn]{glossaries}

\usepackage{listings}
\usepackage{lstbayes}
\lstdefinestyle{mystyle}{
    commentstyle=\color{OliveGreen},
    numberstyle=\tiny\color{black!60},
    stringstyle=\color{BrickRed},
    basicstyle=\ttfamily\scriptsize,
    breakatwhitespace=false,
    breaklines=true,
    captionpos=b,
    keepspaces=true,
    numbers=none,
    numbersep=5pt,
    showspaces=false,
    showstringspaces=false,
    showtabs=false,
    tabsize=2
}
\usepackage{enumitem}

\crefname{lemma}{lemma}{lemmas}
\Crefname{lemma}{Lemma}{Lemmas}
\crefname{theorem}{theorem}{theorems}
\Crefname{theorem}{Theorem}{Theorems}
\crefname{prop}{proposition}{propositions}
\Crefname{prop}{Proposition}{Propositions}

\newtheorem{thm}{Theorem} 
\newtheorem{defn}[thm]{Definition} 
\newtheorem{prop}[thm]{Proposition}

\usepackage[margin=1in]{geometry}
\usepackage{tcolorbox}
\usepackage{adjustbox}

\usepackage[defaultlines=3,all]{nowidow}

\usepackage{libertine}  
\usepackage{dsfont}

\usepackage{caption}
\usepackage{subcaption}
\usepackage{wrapfig}

\DeclareCaptionFormat{empty}{}

\everypar{\looseness=-1}
\title{Goal-Oriented Influence-Maximizing Data Acquisition for Learning and Optimization}

\author{
  Weichi Yao \\
  University of Michigan \\
  weichiy@umich.edu \\
  \and
  Bianca Dumitrascu \\
  Columbia University\\
  bmd2151@columbia.edu \\
  \and 
  Bryan R. Goldsmith \\
  University of Michigan  \\
  bgoldsm@umich.edu\\
  \and
  Yixin Wang \\
  University of Michigan  \\
  yixinw@umich.edu\\
}

\date{\today}

\begin{document}
\maketitle

\begin{abstract}

Active data acquisition is central to many learning and optimization tasks in deep neural networks, yet remains challenging because most approaches rely on predictive uncertainty estimates that are difficult to obtain reliably. To this end, we propose Goal-Oriented Influence-Maximizing Data Acquisition (GOIMDA), an active acquisition algorithm that 
avoids explicit posterior inference while remaining uncertainty-aware through inverse curvature. GOIMDA selects inputs by maximizing their expected influence on a user-specified goal functional, such as test loss, predictive entropy, or the value of an optimizer-recommended design. Leveraging first-order influence functions, we derive a tractable acquisition rule that combines the goal gradient, 
training-loss curvature, and candidate sensitivity to model parameters.
We show theoretically that, for generalized linear models, GOIMDA approximates predictive-entropy minimization up to a correction term accounting for goal alignment and prediction bias, thereby, yielding uncertainty-aware behavior without maintaining a Bayesian posterior. Empirically, across learning tasks (including image and text classification) and optimization tasks (including noisy global optimization benchmarks and neural-network hyperparameter tuning), GOIMDA consistently reaches target performance with substantially fewer labeled samples or function evaluations than uncertainty-based active learning and Gaussian-process Bayesian optimization baselines.

\end{abstract}

Keywords: Deep learning, Active acquisition, Bayesian optimization, Active learning, Influence function 
 
\section{Introduction}
\label{sec:intro}
Active data acquisition is central to many learning and optimization problems in science and engineering, where evaluations or labels are costly, and budgets are limited. A prominent example arises in materials science, where the goal is to discover materials with exceptional properties defined over a vast design space of elemental compositions and atomic structures \citep{Lookman2019ActiveLI,Pilania2021ActiveMS}. Evaluating a single candidate typically requires material synthesis followed by experimental testing, making each function evaluation expensive and time-consuming.

Similar challenges appear in biological system identification, where researchers seek to characterize the response of a system to external stimuli \citep{Sverchkov2017biologyAL}. In this setting, a predictive model—often a neural network—is trained on existing stimulus–response pairs and used to guide the selection of new stimuli that are expected to reduce prediction error. Each new response, however, must be obtained through complex laboratory experiments, making data acquisition the primary bottleneck.

A third example is hyperparameter optimization for deep neural networks \citep{Wu2019PracticalMB}. Hyperparameters govern both model architecture and training dynamics and have a large impact on performance. Evaluating a single configuration requires a full training and validation cycle, which is computationally expensive, especially for large-scale models.

Across these domains, the common challenge is to optimize a task-specific objective using as few evaluations as possible. A natural approach is iterative data acquisition: starting from an initial dataset $\mathcal{D}$ (which can be empty), one repeatedly fits a model $\mathcal{M}$ to the current data and selects the next input $x$ to evaluate by maximizing an acquisition function $a(x)$ that estimates the expected utility of labeling $x$.

When neural networks are used as the underlying model, most existing approaches, spanning Bayesian optimization and deep active learning, rely on predictive uncertainty to guide data selection. For example, in materials design and hyperparameter tuning, Bayesian optimization methods often use posterior predictive uncertainty to balance exploration and exploitation~\citep{Frazier2018ATO,Wang2023review}, while in stimulus--response learning, uncertainty-based criteria dominate active learning practice~\citep{Settles2009ActiveLL,Schrder2020surveyALforText,ren2021surveyDAL, Li2025surveyDALrecent}. Uncertainty estimates therein are typically obtained via Bayesian neural networks~\citep{li2024BNNBOreview} or ensemble-based methods \citep{Frazier2018ATO,Wang2023review,Settles2009ActiveLL,ren2021surveyDAL}.

Despite their popularity, reliable uncertainty estimation for deep neural networks remains a major challenge. Exact Bayesian inference is intractable, and approximate methods are often computationally expensive, sensitive to modeling choices, and poorly calibrated in high-dimensional regimes \citep{Springenberg2016BORobustBNN,li2024BNNBOreview,arbel2023primerBNN}. As a result, uncertainty-driven acquisition functions can be unreliable or impractical, limiting their effectiveness in modern deep learning pipelines.

\

\textbf{Main idea.} To address this challenge, we propose \emph{Goal-Oriented Influence-Maximizing Data Acquisition (GOIMDA)}, 
an iterative acquisition algorithm that 
is uncertainty-aware without explicit posterior inference.
Rather than quantifying predictive uncertainty, 
GOIMDA selects data points based on their expected \emph{influence} on a user-specified \textit{goal objective function} of the model parameters. 
The goal can represent diverse scientific objectives, such as test loss, predictive entropy, or the value of an optimizer-recommended design. In other words, GOIMDA maximizes a goal objective function while minimizing data acquisition efforts.

Using first-order influence functions, we derive a tractable acquisition rule that combines (i) the gradient of the goal objective,  
(ii) an inverse-curvature preconditioner given by the inverse Hessian of the empirical training loss, 
and (iii) each candidate's sensitivity to the model parameters. 
The goal gradient and candidate sensitivity promote goal-directed exploitation by favoring points whose induced parameter updates are most aligned with directions that improve the goal. At the same time, the inverse curvature term acts as a local uncertainty proxy that encourages exploration by prioritizing updates in directions the current data constrains the least. This yields an exploration–exploitation trade-off without maintaining a Bayesian posterior~\citep{efron2015frequentist,giordano2017covariances,basu1996local,alaa2020discriminative}.

More formally, let $(x,y)$ denote a feature--label pair with $x \in \mathcal{X}$ and $y \in \mathcal{Y}$, and let $p_0(y \mid x)$ be the unknown data-generating conditional distribution. We assume that sampling inputs $x$ is free, while observing labels $y$ is costly. Let $\mathcal{M}_\theta$ be a supervised model of $p_\theta(y \mid x)$ with parameters $\theta$, trained by minimizing the negative log-likelihood loss
\[
\ell\big(\theta; (x,y)\big) \;=\; -\,\log p_\theta(y \mid x).
\]

The goal objective function $\mathcal{G}$ can be flexibly defined for various scientific tasks that can be formulated as optimization problems. Without loss of generality, the following discussion focuses on minimizing $\mathcal{G}$. 
We consider the goal objective function $\mathcal{G}$ as a function of $\theta$, which parametrizes and reflects the best knowledge of the unknown distribution $p_0$. 
 
Goal-Oriented Influence-Maximizing Data Acquisition alternates between two steps:
\begin{enumerate}
    \item At each step, we fit a model $p_{\theta}(y\mid x)$ with model parameters $\theta$ by minimizing the empirical risk function on the existing data set $\mathcal{D}$
    \begin{align}
        \theta (\mathcal{D})\; =\;  \arg\min_{\theta}\frac{1}{\vert\mathcal{D}\vert}\sum_{(u,v)\in\mathcal{D}}\ell\big(\theta; (u,v)\big).
    \end{align}
    \item We choose to label the next data point that maximizes the expected influence $\mathcal{I}$ of any candidate $x$ on the goal objective function $\mathcal{G}$
    \begin{align}
        x_{\mathrm{next}}\; =\; \arg\max_{x} \,
        \mathbb E_{y}\big[
        \mathcal{I} \circ \mathcal{G}\big(\theta(\mathcal{D}\cup \{(x,y)\})\big)\big],
    \end{align}
    where the influence is measured by the instantaneous rate at which upweighting $x$ gives the maximum change in $\mathcal{G}$
    \begin{equation}
         \mathcal{I} \circ  \mathcal{G}\big(\theta(\mathcal{D}\cup \{(x,y)\})\big) \, := \, 
         s \cdot \frac{\partial}{\partial\epsilon}\mathcal{G}\big(\theta^{\epsilon}(\mathcal{D}\cup \{(x,y)\})\big)\Big\vert_{\epsilon=0},
         \label{eq:influence_function_on_goal}
    \end{equation}
    where 
    \begin{equation*}
        s\;=\;
         \begin{cases}
             +1,&\text{maximization of }\mathcal{G}\\
             -1,&\text{minimization of }\mathcal{G} \\
         \end{cases}.
    \end{equation*}

    In this formulation, $\theta^{\epsilon}(\mathcal{D}\cup \{(x,y)\})\big)$ is the parameter estimate after $x$ is added to the current dataset $\mathcal{D}$ with weight $\epsilon$, and the output of any candidate $x$ is approximated using resampling techniques such as Jackknife \citep{tukey1958bias}.
\end{enumerate}

\

\begin{figure}[t]
    \centering 
    \includegraphics[scale=0.57]{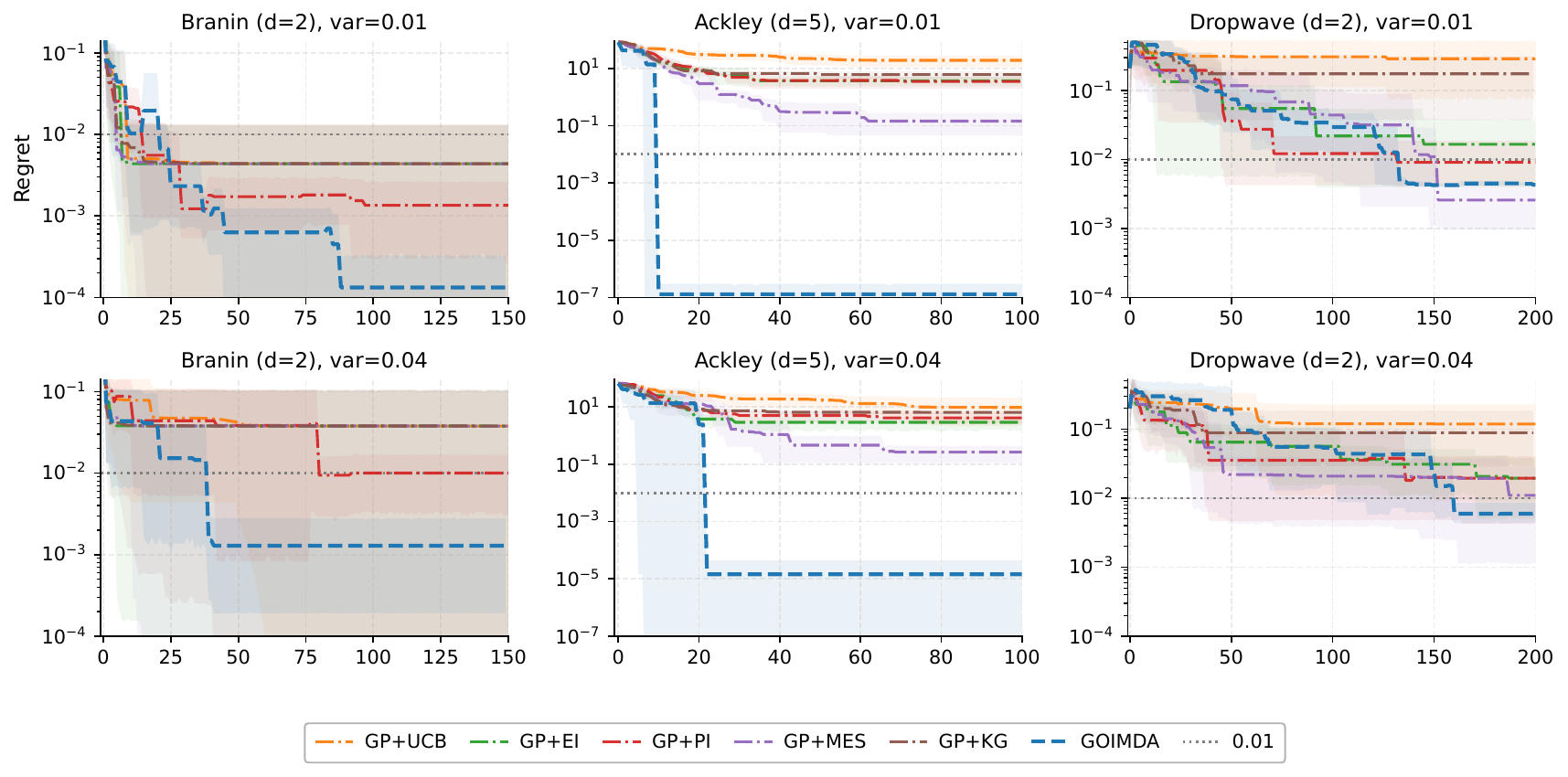}
    \caption{{\small
    \textbf{GOIMDA reaches lower immediate regret with fewer acquisitions than Bayesian optimization baselines on noisy objective functions.} 
    The Bayesian optimization baselines are Gaussian Processes-based with acquisition functions: upper confidence bound (GP+UCB), expected improvement (GP+EI), probability of improvement (GP+PI), max-value entropy search (GP+MES), and knowledge gradient (GP+KG).  
    Immediate regret is reported at each acquisition step for the Branin, Ackley, and Dropwave benchmarks under two noise levels ($\sigma^2=0.01,0.04$). 
    Solid/dashed curves show the mean performance across runs, and shaded regions denote bootstrapped 95\% confidence intervals of the mean (computed across runs). 
    Across all tasks, GOIMDA consistently achieves lower regret earlier in the acquisition process, with the advantage generally becoming more pronounced at higher noise levels.}}
    \label{fig:ao_noisy_function}
\end{figure}

\textbf{Contributions.} We introduce GOIMDA, a general active data acquisition framework that selects new queries by maximizing their expected first-order influence on a user-specified goal functional (e.g., test loss, predictive entropy, or the value of an optimizer-recommended design). We derive a tractable influence-based acquisition rule that couples the goal gradient, training-loss curvature (via inverse-Hessian–vector products), and candidate sensitivity.
We provide a scalable implementation using a Jackknife/deep-ensemble surrogate to approximate unknown labels and stochastic inverse-Hessian vector product solvers to handle modern neural networks.

We further develop theoretical results under exponential-family models, showing that GOIMDA approximates predictive-entropy minimization up to a correction term that accounts for goal alignment and prediction bias. This result explains how GOIMDA captures uncertainty-related behavior 
through the same curvature structure that underlies predictive entropy, while modulating it via goal alignment and prediction bias, but without maintaining a posterior over $\theta$.

Finally, we demonstrate empirically across a range of learning and optimization tasks, including image and text classification, noisy global optimization benchmarks, and neural-network hyperparameter tuning, that GOIMDA consistently achieves target performance with substantially fewer labeled samples or function evaluations than uncertainty-based active learning algorithms and Gaussian-process Bayesian optimization baselines. 
See Figure~\ref{fig:ao_noisy_function} for an example, where we benchmark GOIMDA on noisy black-box function optimization from \citep{bingham2015testfunc} with five other commonly used Gaussian Processes (GP)-based Bayesian optimization methods with different acquisition functions, namely, probability of improvement (PI) \citep{Torn1989PI}, expected improvement (EI) \citep{jones1998EI}, upper confidence bound (UCB) \citep{Srinivas2010UCB}, max-value entropy search (MES) \citep{wang2017MES}, and knowledge gradient (KG) \citep{scott2011KGgaussianmodel,frazier2008KLdiscrete,frazier2009KLcorrelatednormal}; more details provided in Section~\ref{sec:emipirical_studies} for empirical studies.

\textbf{Organization.} The rest of the paper is organized as follows: Section \ref{sec:method} presents the GOIMDA algorithm and its scalable implementation.
Section \ref{sec:exponential} develops theory under exponential-family models and connects GOIMDA to predictive-entropy-based acquisition.
Section \ref{sec:emipirical_studies} reports empirical results on learning and optimization benchmarks.
It concludes with a discussion of related work in Section \ref{sec:related}, along with the scope and limitations in Section \ref{sec:discussion}.


\section{Goal-Oriented Influence-Maximizing Data Acquisition}
\label{sec:method}
We introduce \emph{goal-oriented influence-maximizing data acquisition (GOIMDA)}. Unlike standard acquisition strategies that optimize proxy criteria (e.g., predictive uncertainty alone), GOIMDA allows the user to specify an explicit \emph{goal objective}
\[
\mathcal{G}(\theta(\mathcal{D})),
\]
which encodes the scientific task of interest and depends on the model parameters \(\theta(\mathcal{D})\) learned from the currently labeled dataset \(\mathcal{D}\). At each iteration, the objective is to acquire the data point whose label is expected to yield the largest improvement in this goal.

A naïve implementation would retrain the model for every candidate \(x\) to evaluate the post-update value of \(\mathcal{G}\), which is computationally infeasible. We therefore develop an \emph{influence-function-based approximation}. Using classical influence functions, we estimate the \emph{first-order change} in \(\mathcal{G}\) induced by infinitesimally upweighting a candidate example. This yields a tractable acquisition score that couples: (i) the goal gradient \(\nabla_\theta \mathcal{G}\), (ii) a curvature preconditioner given by the inverse Hessian of the empirical training loss, and (iii) the candidate's parameter sensitivity \(\nabla_\theta \ell(\theta; (x, y))\).

To handle unknown candidate labels and to scale to modern neural networks, we further introduce (a) a surrogate ensemble to approximate expectations over unknown labels and (b) scalable implicit inverse--Hessian--vector products, avoiding explicit construction of \(H_\theta^{-1}\). The remainder of this section specifies representative goal objectives \(\mathcal{G}\), derives the influence-function-based acquisition rule, and presents practical algorithms for efficient implementation.

\subsection{Goal Objective Function}

We consider the goal objective \(\mathcal{G}\) defined as a function of the current fitted parameters \(\theta := \theta(\mathcal{D})\), which summarize the most up-to-date model induced by the acquired dataset \(\mathcal{D}\). This formulation allows \(\mathcal{G}\) to be flexibly instantiated to match different scientific tasks.

\textbf{Global optimization.}
Consider minimizing an unknown function \(f\), whose observations take the form \(y = f(x) + \epsilon\). This setting covers applications such as materials design and neural-network hyperparameter tuning. Let \(x^\ast = \arg\min_x f(x)\) denote the true minimizer, which is unknown. We define the goal objective $\mathcal{G}$ as 
\begin{align}
    \mathcal{G}^{\text{opt}}(\theta)
    \;:=\; \mathbb{E}_{y \sim p_0(\cdot \mid \hat{x}^\ast_\theta)}[\, y \,],
    \label{eq:goal_opt}
\end{align}
the expected property value at the model's recommended minimizer  $\hat{x}^\ast_\theta$, obtained by minimizing model's predictive mean $\mathbb E_{y\sim p_\theta(\cdot\mid x)}[\,y\,]$.

\textbf{Targeted supervised learning.}
For supervised learning problems where performance is evaluated on an unlabeled target set \(\mathcal{U}\) that shares the same conditional distribution \(p_0(y \mid x)\), we define \(\mathcal{G}\) as a utility function over \(\mathcal{U}\). A common choice is the negative log-likelihood:
\begin{align}
    \mathcal{G}^{\text{nll}}(\theta)
    \;:=\; -\sum_{x \in \mathcal{U}} \mathbb{E}_{y \sim p_0(\cdot \mid x)}
    [\log p_\theta(y \mid x)].
    \label{eq:goal_nll}
\end{align}
When the objective places greater importance on reducing errors for hard or low-confidence predictions, as opposed to uniformly penalizing all examples, we adopt the focal loss \citep{lin2020focalloss}:
\begin{align}
    \mathcal{G}^{\text{foc}}(\theta)
    \;:=\; -\sum_{x \in \mathcal{U}} \mathbb{E}_{y \sim p_0(\cdot \mid x)}
    \big[(1 - p_\theta(y \mid x))^\gamma \log p_\theta(y \mid x)\big],
    \label{eq:goal_focal}
\end{align}
where \(\gamma\) is a predefined focused parameter or the relaxation parameter.

\textbf{Entropy-based objectives.}
If the goal is to reduce predictive uncertainty on \(\mathcal{U}\), we use the Shannon entropy of the model's predictive distribution:
\begin{align}
    \mathcal{G}^{\text{ent}}(\theta)
    \;:=\; -\sum_{x \in \mathcal{U}}
    \mathbb{E}_{y \sim p_\theta(\cdot \mid x)}
    [\log p_\theta(y \mid x)].
    \label{eq:goal_ent}
\end{align}
Although \(\mathcal{G}^{\text{ent}}\) and \(\mathcal{G}^{\text{nll}}\) share the same functional form, they differ in the underlying expectation: the former is taken with respect to the model distribution \(p_\theta\), while the latter is evaluated under the true conditional distribution \(p_0\).

Overall, these examples illustrate that the goal objective \(\mathcal{G}\) can be adapted to a wide range of tasks. Without loss of generality, we focus on settings where the objective is to \emph{minimize} \(\mathcal{G}\).

\subsection{Influence Function}

\subsubsection{Goal-based informativeness}

Given a goal minimization objective \(\mathcal{G}\), we call a candidate input \(x_{\mathrm{c}}\) \emph{informative} if acquiring its label is expected to affect the model update and hence reduce \(\mathcal{G}\) the most. 

Intuitively, a candidate data point is informative if its inclusion would substantially alter the learned model parameters and, as a result, lead to a significant reduction in the goal objective. This notion is formalized below.

\begin{defn}[Informativeness]\label{def:informativeness}
Given a goal objective \(\mathcal{G}\) and a current dataset \(\mathcal{D}\), the informativeness of a candidate input \(x_{\mathrm{c}}\) is measured by the change in the goal objective after adding \(x_{\mathrm{c}}\) to \(\mathcal{D}\):
\begin{align}
    \Delta \mathcal{G} 
    \;:=\; \mathcal{G}\big(\theta(\mathcal{D} \cup \{(x_{\mathrm{c}}, y_{\mathrm{c}})\})\big) 
    - \mathcal{G}(\theta(\mathcal{D})).
    \label{eq:reduction_in_goal}
\end{align}
Here, \(\theta(\mathcal{D})\) and \(\theta(\mathcal{D} \cup \{(x_{\mathrm{c}}, y_{\mathrm{c}})\})\) denote the parameters before and after acquiring \(x_{\mathrm{c}}\), respectively. 
\end{defn}

Although this definition is conceptually clear, directly evaluating \(\Delta \mathcal{G}\) is computationally infeasible, as it requires retraining the model for every candidate \(x_{\mathrm{c}}\) and every possible label \(y_{\mathrm{c}}\). In the following, we introduce an influence-function-based approximation \citep{huber1972influence} that enables efficient estimation of informativeness without retraining the model.

\subsubsection{Influence function approximation}
We define the influence of a candidate point \((x_{\mathrm{c}}, y_{\mathrm{c}})\) on the goal objective \(\mathcal{G}\) as the instantaneous change in \(\mathcal{G}\) induced by infinitesimally upweighting this point in the training objective; see (\ref{eq:influence_function_on_goal}). In the case of minimizing the goal objective $\mathcal{G}$, the influence defined takes the instantaneous rate with a negative sign
\begin{equation}
    \mathcal{I} \circ \mathcal{G}\big(\theta(\mathcal{D}\cup \{(x_{\mathrm{c}},y_{\mathrm{c}})\})\big)\;=\;-\,\frac{\partial}{\partial\epsilon} \mathcal{G}\big(\theta^{\epsilon}(\mathcal{D}\cup \{(x_{\mathrm{c}},y_{\mathrm{c}})\})\big)  \big\vert_{\epsilon=0},
    \label{eq:influence_function_on_goal_minimization}
\end{equation}
where $\theta^{\epsilon}(\mathcal{D}\cup \{(x_{\mathrm{c}},y_{\mathrm{c}})\})$ is the parameter estimate when $x_{\mathrm{c}}$ is added to $\mathcal{D}$ with weight $\epsilon$:
\begin{align}
    \theta^{\epsilon}(\mathcal{D}\cup \{(x_{\mathrm{c}},y_{\mathrm{c}})\})
    \;=\; \arg\min_{\theta} \frac{1}{|\mathcal{D}|}\sum_{(u,v)\in\mathcal{D}}\ell\big(\theta; (u, v)\big) + \epsilon\,\ell\big(\theta; (x_{\mathrm{c}}, y_{\mathrm{c}})\big).
    \label{eq:theta_epsilon}
\end{align}

\textbf{Maximizing the influence score for maximizing informativeness. } The following result shows that maximizing the reduction in $\mathcal{G}$ when candidate $x_{\text{c}}$ is added in the training set $\mathcal{D}$ can be linearly approximated by maximizing $\mathcal I \circ\mathcal{G} (\theta(\mathcal{D}\cup \{(x_{\text{c}},y_{\text{c}})\}))$ without retraining the model.

\begin{prop}\label{prop:minIequivmaxRh}
    The influence score of a candidate point $(x_{\mathrm{c}}, y_{\mathrm{c}})$ on the goal minimization objective $\mathcal{G}$, defined in (\ref{eq:influence_function_on_goal_minimization}), approximates the reduction in $\mathcal{G}$ after adding $x_{\mathrm c}$ to $\mathcal{D}$. 
\end{prop} 
\begin{proof}
    By definition of $\Delta \mathcal{G}$ in (\ref{eq:reduction_in_goal}), $\Delta \mathcal{G}$ is negative for minimization of the goal objective $\mathcal{G}$. 
    To maximize the reduction in $\mathcal G$, we maximize $-\Delta \mathcal{G}$. 
    Since applying the first-order Taylor expansion on the goal objective function $\mathcal{G}\big(\theta^\epsilon(\mathcal{D}\cup\{(x_{\text{c}}, y_{\text{c}})\})\big)$ at $\epsilon = 0$ gives 
    \begin{align}
        \mathcal{G}\big(\theta^\epsilon(\mathcal{D}\cup\{(x_{\text{c}}, y_{\text{c}})\})\big)\;=\; \mathcal{G}(\theta(\mathcal{D}))  +\Big[ \frac{\partial }{\partial \epsilon} \mathcal{G}\big(\theta^\epsilon(\mathcal{D}\cup\{(x_{\text{c}}, y_{\text{c}})\}\big)\Big]\Big\vert_{\epsilon=0} \cdot\epsilon + \mathrm{o}(\epsilon^2),  
    \end{align}
    we have 
    \begin{align}
         - \,\Delta \mathcal{G}  \;\approx\; -
         \,\frac{1}{\vert \mathcal{D}\vert }\Big[ \frac{\partial }{\partial \epsilon} \mathcal{G}\big(\theta^\epsilon(\mathcal{D}\cup\{(x_{\text{c}}, y_{\text{c}})\}\big)\Big]\Big\vert_{\epsilon=0}  \;\stackrel{(\ref{eq:influence_function_on_goal_minimization})}{=}\;\frac{1}{\vert \mathcal{D}\vert } \;\mathcal{I} \circ \mathcal{G}\big(\theta(\mathcal{D}\cup \{(x_{\text{c}},y_{\text{c}})\})\big)
         .  
    \end{align}
    Therefore, we can linearly approximate the reduction in test loss due to adding $x_\mathrm{c}$ without retraining the model by computing $\frac{1}{\vert \mathcal{D}\vert }\,\mathcal{I}\circ \mathcal{G}$. 
\end{proof}

Under a minimization objective, the most informative data point is the one that induces the greatest reduction in the goal objective. Proposition~\ref{prop:minIequivmaxRh} shows that the reduction can be linearly approximated via the influence score. As a result, selecting the most informative data point reduces to choosing the candidate with the maximum influence \citep{koh2017understanding}. 
In this precise sense, the influence-score-based acquisition criterion is well justified: it corresponds to the greedy choice that most effectively decreases \(\mathcal{G}\) under a first-order influence approximation.

\subsubsection{Closed-form influence score}

We now derive a closed-form expression for the influence score, which yields a tractable acquisition criterion. The explicit form of (\ref{eq:influence_function_on_goal_minimization}) can be derived by first applying the chain rule:
\begin{align}
    \mathcal{I} \circ \mathcal{G}\big(\theta(\mathcal{D}\cup \{(x_{\text{c}},y_{\text{c}})\})\big)\;=\;-\,\Big[\nabla_{\theta}\,\mathcal{G}\big(\theta(\mathcal{D})\big)\Big]^\top\cdot\Big[\frac{\partial}{\partial\epsilon}\theta^{\epsilon}\big(\mathcal{D}\cup(\{x_{\text{c}},y_{\text{c}}\})\big)\Big]\Big\vert_{\epsilon=0},
    \label{eq:influence_loss_params}
\end{align}

To evaluate the parameter sensitivity term, we apply standard influence-function arguments. Using the first-order optimality condition of $\theta^{\epsilon}(\mathcal{D}\cup \{(x_{\text{c}},y_{\text{c}})\})$ and the fact that $\theta^\epsilon(\mathcal{D}\cup \{(x_{\text{c}},y_{\text{c}})\})\rightarrow \theta (\mathcal{D})$ as $\epsilon\rightarrow 0$ \citep{Influence_book}, a classical Taylor expansion on $\frac{\partial }{\partial \epsilon}\theta^{\epsilon}(\mathcal{D}\cup \{(x_{\text{c}},y_{\text{c}})\})$ in (\ref{eq:influence_loss_params}) at $\epsilon=0$ gives
\begin{align}
    \frac{\partial}{\partial\epsilon}\theta^{\epsilon}\big(\mathcal{D}\cup(\{x_{\text{c}},y_{\text{c}}\})\big)\Big\vert_{\epsilon=0}\;\approx\; -\,H_{\theta}^{-1}\,\Big[\nabla_{\theta}\, \ell\big( \theta;(x_{\text{c}},y_{\text{c}})\big) \Big]\Big\vert_{\theta=\theta(\mathcal{D})},
    \label{eq:influence_params}
\end{align}
where $H_{ \theta }$ is the Hessian matrix
\begin{align}
    H_{ \theta }\;:=\;\frac{1}{\vert \mathcal{D}\vert }\sum_{(u,v)\in \mathcal{D}}\nabla_\theta^2
    \,\ell\big( \theta; (u,v)\big).
    \label{eq:hessian}
\end{align}

Substituting~\eqref{eq:influence_params} into~\eqref{eq:influence_loss_params} yields a closed-form approximation $\tilde{\mathcal{I}}\circ \mathcal{G}\big( \theta(\mathcal{D}\cup\{(x_{\text{c}},y_{\text{c}})\})\big)$ for (\ref{eq:influence_loss_params}).
Since the candidate label \(y_{\mathrm{c}}\) is unknown at acquisition time, we take expectation with respect to the true conditional distribution \(p_0(y \mid x_{\mathrm{c}})\), obtaining the expected influence score
\begin{equation}
    \begin{aligned}
         \mathbb E_{y_{\text{c}}\sim p_0(\cdot\mid x_{\text{c}})}\big[\tilde{\mathcal{I}} \circ \mathcal{G}\big( \theta(\mathcal{D}\cup\{(x_{\text{c}},y_{\text{c}})\})\big)\big] \;=\;\Big[\nabla_{\theta}\,\mathcal{G}\big(\theta(\mathcal{D})\big)\Big]^\top H_{\theta}^{-1}\,\Big[\nabla_{\theta}\mathbb E_{y_{\text{c}}\sim p_0(\cdot\mid x_{\text{c}})}\ell\big( \theta;(x_{\text{c}},y_{\text{c}})\big)\Big]\Big\vert_{\theta=\theta(\mathcal{D})}.
    \end{aligned}\label{eq:closed_form_I}
\end{equation}

Equation~\eqref{eq:closed_form_I} provides an explicit, first-order approximation to the influence of a candidate point on the goal objective \(\mathcal{G}\), without requiring access to the retrained parameters
\(\theta(\mathcal{D} \cup \{(x_{\mathrm{c}}, y_{\mathrm{c}})\})\).
Consequently, the most informative candidate can be selected by maximizing~\eqref{eq:closed_form_I}. This yields an efficient goal-oriented data acquisition rule that directly targets reductions in \(\mathcal{G}\) while avoiding the computational cost of retraining the model for every candidate.

\subsection{Practical and Scalable Implementation}

By definition, the influence score in~(\ref{eq:influence_function_on_goal_minimization}) with a closed-form expression~(\ref{eq:closed_form_I}) traces how the optimization of the goal objective
$\mathcal{G}$ propagates through the learned parameters $\theta$ and back to individual
training data points, thereby identifying the point with the greatest potential to reduce
$\mathcal{G}$. Proposition~\ref{prop:minIequivmaxRh} further shows that maximizing the
influence score provides an efficient approximation to maximizing
the informativeness in terms of reduction in $\mathcal{G}$.

However, computing the closed-form influence score in~(\ref{eq:closed_form_I}) is nontrivial.
First, evaluating the influence of any candidate input on the goal
objective typically requires expectations over the true conditional distribution $p_0(y\mid x)$,
which is unavailable in practice. Second, as the number of acquired data points
$\left\vert \mathcal{D} \right\vert$ grows, directly computing the inverse Hessian
$H_\theta^{-1}$ defined in~(\ref{eq:hessian}) becomes prohibitively expensive, making
naive influence-based acquisition infeasible. In this section, we introduce practical
and scalable algorithms that address both challenges and enable efficient computation
of the influence function in~(\ref{eq:closed_form_I}).

\textbf{Approximating the unknown output variable.} Goal objectives $\mathcal{G}(\theta)$ are often defined on unknown target outputs, such
as~(\ref{eq:goal_opt}) for iterative global optimization and~(\ref{eq:goal_nll}) for active
learning. At the same time, the influence score is
evaluated at $\theta(\mathcal{D} \cup \{(x_{\mathrm{c}}, y_{\mathrm{c}})\})$, where the
candidate output $y_{\mathrm{c}}$ is unknown prior to acquisition. Consequently,
computing the influence score of any candidate $x_{\mathrm{c}}$ requires an approximation of the goal objective $\mathcal{G}$ with potentially unknown target outputs and 
the corresponding output $y_{\mathrm{c}}$.

Most existing work that applies influence functions to iterative data acquisition
operates within an active learning framework
\citep{xu2019understandinggoalorientedactivelearning,wang2022boostingactivelearningimproving}.
One line of work selects goal objectives, such as negative prediction entropy
(\ref{eq:goal_ent}) or variants of the Fisher Information Ratio \citep{Zhang2000FIR}, which
depend only on the current model $p_\theta$ and do not explicitly target the test
performance under the true data-generating distribution
\citep{xu2019understandinggoalorientedactivelearning}. Another approach derives
acquisition functions from upper bounds on the test loss, leading to criteria based on
the gradient norm
$\lVert \nabla_\theta \mathbb{E}_{y_{\mathrm{c}}\sim p_\theta(\cdot\mid x_{\mathrm{c}})}
[\ell (\theta; (x_{\mathrm{c}}, y_{\mathrm{c}}) )] \rVert$
\citep{wang2022boostingactivelearningimproving}. In both cases, the unknown candidate outputs
$y_{\mathrm{c}}$ are approximated using the current model posterior $p_\theta$, relying
solely on training information.

In contrast, our objective is to directly optimize goal functions defined under the
unknown true distribution $p_0$, such as test loss in active learning or the minimal
value in iterative global optimization. To this end, we approximate the unknown output
without relying on the posterior of the primary model $p_\theta$.

Specifically, we introduce a surrogate model $\tilde{\mathcal{M}}_\phi$ with model parameter $\phi$, implemented as
an ensemble of $r$ neural networks with different random initializations and trained on
different subsets of the training data. This approach is closely related to deep
ensembles, which have been shown to improve predictive accuracy, uncertainty estimation,
and robustness to distributional shift \citep{fort2020deepensembleslosslandscape}. To
further enhance predictive stability, we employ the Jackknife resampling technique
\citep{tukey1958bias}. Each ensemble member is trained on a Jackknife subsample of the
available data, and the surrogate model outputs the average prediction across all $r$
networks. The resulting Jackknife estimation has computational complexity $O(rm)$,
where $m$ is the number of model parameters.


\textbf{Computing the inverse Hessian--vector product.} Evaluating (\ref{eq:closed_form_I}) requires
computing the inverse of the Hessian matrix defined in~(\ref{eq:hessian}). For a model
$\mathcal{M}$ with $m$ trainable parameters and a dataset $\mathcal{D}$ of size $n$,
explicitly forming and inverting the Hessian incurs a computational cost of
$O(nm^2 + m^3)$, which is infeasible for modern neural networks with millions of
parameters.

To avoid explicit Hessian inversion, we approximate inverse Hessian--vector products
(HVPs), which can be computed in linear time with respect to the number of parameters
using automatic differentiation frameworks such as \texttt{PyTorch} and \texttt{JAX}.
A common approach is to use conjugate gradient (CG) methods \citep{Mehra1969cg_ihvp},
which solve
\[
\min_u \; u^\top (H^2 + \lambda I) u - (Hv)^\top u,
\]
where the solution $u^\ast$ approximates $H^{-1}v$. With $L$ CG iterations, the
computational cost is $O(nL)$, and in practice, convergence can often be achieved with a
small number of iterations \citep{Martens2010cg_speedup}.

When the training dataset is very large, the linear dependence on $n$ can still be
costly. Alternatively, we adopt the LiSSA algorithm \citep{agarwal2017second}, which
recursively estimates the inverse HVPs via
\[
\hat{H}_j^{-1} v \;=\; v + (I - H)\hat{H}_{j-1}^{-1} v,
\qquad
\hat{H}_0^{-1} v \;=\; v.
\]
LiSSA stochastically approximates inverse HVPs using mini-batches of data at each
iteration. With batch size $B$ and $L$ iterations, the computational complexity is
$O(nm + BLm)$. Both CG and LiSSA are well-suited to large-scale settings with limited
memory, and empirically, we observe no significant difference between them in terms of
acquisition performance.

\textbf{Goal-Oriented Influence-Maximizing Data Acquisition (GOIMDA): A general efficient iterative acquisition algorithm.} In practice, we select the next acquisition point by
\begin{align}
    x_{\mathrm{next}}\;=\; \arg\max_{x_{\text{c}}}\;
    \mathbb E_{y_{\text{c}}\sim p_{\phi}(\cdot\mid x_{\text{c}})}\big[\tilde{\mathcal{I}} \circ \mathcal{G}_{\phi}\big(\theta(\mathcal{D}\cup \{(x_{\text{c}}, y_{\text{c}})\})\big)\big],
    \label{eq:x_next} 
\end{align}
where the influence term admits the explicit approximation
\begin{align}
    \tilde{\mathcal{I}} \circ  \mathcal{G}_{\phi}\big(\theta(\mathcal{D}\cup \{(x_{\text{c}},y_{\text{c}})\})\big)\;=\;\Big[\nabla_{\theta}\,\mathcal{G}_{\phi}\big(\theta(\mathcal{D})\big)\Big]^\top \hat{H}_{\theta}^{-1}\,\Big[\nabla_{\theta}\,\ell\big( \theta;(x_{\text{c}}, y_{\text{c}})\big)\Big]\Big\vert_{\theta=\theta(\mathcal{D})}. 
    \label{eq:closed_form_I_practice}
\end{align}
Here, $\mathcal{G}_{\phi}$ denotes the goal objective approximated using the surrogate model
$\tilde{\mathcal{M}}_\phi$ with model parameter $\phi$ in the absence of direct access to $p_0$, and
$\hat{H}_\theta^{-1}$ denotes a stochastic approximation of the inverse Hessian obtained
via implicit HVPs. See 
Algorithm~\ref{algo:GOI_general} for the complete algorithm.

\begin{algorithm}[t]
\caption{Goal-Oriented Influence-Maximizing Data Acquisition} 
\begin{algorithmic}[1]
\INPUT The initial data set $\mathcal{D}$
\REPEAT 
\STATE Update both main model $\mathcal{M}_\theta$ and the ensemble model $\tilde{\mathcal{M}}_{\phi}$ on $\mathcal{D}$  
\STATE Select $x_{\mathrm{next}}\gets\arg\max_{x}\;\mathbb E_{y\sim p_\phi(\cdot\mid x)}\big[\tilde{\mathcal{I}} \circ \mathcal{G}_{\phi}\big(\theta(\mathcal{D}\cup \{(x,y)\})\big)\big]$ \hfill\COMMENT{see approximation in (\ref{eq:closed_form_I_practice})}
\STATE Query $y_{\mathrm{next}} \gets \texttt{O{\small BSERVE}}(x_{\mathrm{next}})$  
\STATE Augment $\mathcal{D} \gets \mathcal{D}\cup \{(x_{\mathrm{next}},y_{\mathrm{next}})\}$ 
\UNTIL{termination condition is met} \hfill \COMMENT{e.g. budget exhausted or desired goal achieved}
\end{algorithmic} \label{algo:GOI_general}
\end{algorithm}

\section{Theoretical Properties of GOIMDA under Exponential Family Models} 
\label{sec:exponential}

In this section, we study goal-oriented influence-maximizing data acquisition (GOIMDA)
under exponential family models. We first formalize the exponential-family setting and
derive GOIMDA in closed form. We then provide a geometric interpretation of the resulting
influence function, clarify the role of the bias term, and compare GOIMDA with
predictive-entropy–based acquisition.

Our study reveals that GOIMDA decomposes naturally into three interacting components:
(i) \emph{goal alignment}, captured by the gradient $\nabla_\theta \mathcal{G}$;
(ii) \emph{curvature preconditioning}, governed by the inverse Hessian $H_\theta^{-1}$ of the \emph{empirical training loss}; and
(iii) a \emph{prediction-bias term} that quantifies the discrepancy between the model and
the true data-generating mechanism. Together, these components define an acquisition
criterion that is \emph{uncertainty-aware in a local,
curvature-based sense}: the inverse-curvature term prioritizes candidates that induce updates along
parameter directions that are not yet well constrained by the current data.

For canonical Generalized Linear Models (GLMs), we make this connection explicit by showing that GOIMDA inherits the same $H_\theta^{-1}$ leverage factor that appears in predictive-entropy acquisition, while reweighted by goal alignment and prediction bias.
Moreover, we derive a parameter-space surrogate for the bias term that replaces the unknown true conditional distribution with an estimable parameter discrepancy, yielding a fully tractable acquisition rule.

Overall, these results position GOIMDA as an \emph{exploration-aware, bias-focused
improvement surrogate}: it targets regions that are both influential for the goal and
informative about model misspecification, without requiring a posterior distribution
over~$\theta$.

\textbf{Model setup.} Let $x \in \mathbb{R}^d$ and $y \in \mathbb{R}$. We model the conditional distribution
$y \sim p_\theta(\cdot \mid x)$ using an exponential family of the form
\begin{align}
    p_\theta(y\mid x) \;=\; h(y)\exp\Big (\eta_\theta(x)\, T(y)-A\big(\eta_\theta(x)\big)\Big),
    \label{eq:exponential_family_main}
\end{align}
where $T(y)$ is the sufficient statistic and $\eta_\theta(x)$ is the natural parameter,
parameterized by $\theta$.\footnote{We do not explicitly distinguish between scalar and
vector-valued parameters in the notation.} Exponential family distributions enjoy
several well-known statistical properties:
\begin{enumerate}[label=(\ref{sec:exponential}.\arabic*)]
    \item the log-partition function $A(\eta\big)$ is convex;
    \item $\mathbb{E}_\eta[T(y)] = A^\prime(\eta)$;
    \item $\mathrm{var}_{\eta}[T(y)] = A^{\prime\prime}(\eta)$.
\end{enumerate}

In this setting, the natural parameter $\eta_\theta(x)$ is modeled by a deep neural
network with parameters $\theta$, trained via the negative log-likelihood loss $\ell (\theta; (x,y)):=-\log p_\theta (y\mid x)$, explicitly,
\begin{align}
    \ell\big(\theta;(x,y)\big)
    \;=\;-\log h(y) -\eta_\theta(x)\,T(y)+A(\eta_\theta(x)).\label{eq:exp_loss_main}
\end{align}

\textbf{Influence score.} Using the loss definition for $\ell(\theta;(x,y))$ in~(\ref{eq:exp_loss_main}), the gradient of the expected loss
at a candidate input $x_{\mathrm{c}}$ under the true data-generating distribution
$p_0(\cdot \mid x_{\mathrm{c}})$ is given by
\begin{align*}
    \nabla_{{\theta}}\,\mathbb E_{y_{\text{c}}\sim p_0(\cdot\mid x_{\text{c}})}\big[\ell\big(\theta;(x_{\text{c}},y_{\text{c}})\big)\big] \;=\; \nabla_\theta \eta_\theta(x_{\text{c}}) \big[A^\prime\big(\eta_\theta(x_{\text{c}})\big) - A^\prime\big(\eta_0(x_{\text{c}}))\big)\big].
\end{align*}
Substituting this expression into the closed-form influence score
(\ref{eq:closed_form_I}) yields the following goal-oriented influence (GOI) score:
\begin{align}
    \mathsf{GOI}(x_{\text{c}})\;=\;\big[\nabla_{\theta}\,\mathcal{G}(\theta)\big]^\top  H_{\theta}^{-1} \; \nabla_\theta\,\eta_\theta(x_{\text{c}})\,\big[A^\prime\big(\eta_\theta(x_{\text{c}})\big) - A^\prime\big(\eta_0(x_{\text{c}}))\big)\big].
    \label{eq:goi_exp_main} 
\end{align}
When $T(y)=y$, the term $A^\prime\big(\eta_\theta(x_{\text{c}})\big) - A^\prime\big(\eta_0(x_{\text{c}}))\big)$ in (\ref{eq:goi_exp_main}) reduces to $\mathbb{E}_\theta[y\!\mid\!x_{\text c}]-\mathbb{E}_{0}[y\!\mid\!x_{\text c}]$, which corresponds to the \emph{prediction bias} at $x_{\mathrm{c}}$. Consequently, the
next point to acquire is selected as 
\begin{equation*}
    x_{\mathrm{next}} \;=\; \arg\max_{x_{\mathrm{c}}} \; \mathsf{GOI}(x_{\mathrm{c}}). 
\end{equation*}

\textbf{Geometric interpretation of the influence function.} The acquisition criterion in~(\ref{eq:goi_exp_main}) selects the candidate whose
\emph{curvature-aware} parameter perturbation
is most aligned with the direction that
maximally decreases $\mathcal{G}$, while prioritizing regions where the model is currently biased.

Define
$\mu:= H_\theta^{-1/2}\,\nabla_\theta\mathcal G(\theta)$, $\nu(x_{\mathrm c}):=H_\theta^{-1/2}\,\nabla_\theta\eta_\theta(x_{\mathrm c})$,
and $b(x_{\mathrm c}) := A^\prime(\eta_\theta(x_{\mathrm c})) - A^\prime(\eta_0(x_{\mathrm c}))$.
With these definitions, the goal-oriented influence score can be written as
\begin{equation}
    \mathsf{GOI}(x_{\mathrm c})
    \;=\;\underbrace{\langle \mu,\;\nu(x_{\mathrm c})\rangle}_{\substack{\textbf{goal alignment}\\\text{in the $H_\theta$-geometry}}}\;\times\;
    \underbrace{b(x_{\mathrm c})}_{\substack{\textbf{prediction bias}\\ \mathbb E_\theta[y\mid x_{\mathrm c}]-\mathbb E_{0}[y\mid x_{\mathrm c}]}},
    \label{eq:goi_exp_breakdown}
\end{equation}
where the first factor quantifies how strongly a curvature-aware update in the direction suggested by $x_{\mathrm c}$ is expected to decrease $\mathcal G$, and then the second factor gates this directional effect by the magnitude of the local predictive bias.

Concretely, $\mu$ acts as a sensitivity vector that encodes which parameter directions matter for improving $\mathcal{G}$, 
whereas $\nu(x_{\mathrm{c}})$ characterizes the candidate's induced update direction after accounting for local curvature.
The preconditioning by $H_\theta^{-1}$ transforms Euclidean gradients into second-order–aware directions, down-weighting parameter components that are already well-determined by existing data and amplifying directions where the model remains uncertain or weakly identified; this yields an approximately reparameterization-invariant notion of alignment and prioritizes candidates that can meaningfully refine the model in parts that are not yet well determined. 
Finally, the multiplicative bias factor $b(x_{\mathrm c})$ converts this leverage in parameter space into expected improvement of the goal. That is, if the model is already well-calibrated at $x_{\mathrm c}$ then $b(x_{\mathrm c})\approx 0$ and the predicted first-order gain is negligible, whereas a larger $|b(x_c)|$ steers acquisition toward regions where the predictive mean deviates from the data-generating process and parameter updates translate into tangible reductions of $\mathcal{G}$.

\textbf{The bias term.} To obtain a tractable approximation of the bias term, we apply a first-order Taylor
expansion of $A'\big(\eta_0(x_{\mathrm c})\big)$ around $\theta_0=\theta$, yielding
\begin{align}
    A^\prime\big(\eta_\theta(x_{\text{c}})\big) - A^\prime\big(\eta_0(x_{\text{c}}))\big) \;\approx\; \,\nabla_\theta\eta_\theta(x_{\text{c}})^\top (\theta-\theta_0)\,A^{\prime\prime}\big(\eta_\theta(x_{\text{c}})\big). 
    \label{eq:bias_term_taylor}
\end{align} 
Substituting this approximation into~\eqref{eq:goi_exp_breakdown} gives
\begin{equation}
    \begin{aligned}
    \mathsf{GOI}(x_{\mathrm c})
    \; =\;\underbrace{\langle \mu,\;\nu(x_{\mathrm c})\rangle}_{\substack{\textbf{goal alignment}\\\text{in the $H$-geometry}}}\;\times\;\underbrace{\nabla_\theta \eta_\theta(x_{\text{c}})^\top
    (\theta-\theta_0)}_{\substack{\textbf{directional parameter bias at $x_\mathrm{c}$}}}\,
    A^{\prime\prime}\big(\eta_\theta(x_{\text{c}})\big). 
    \end{aligned}
    \label{eq:goi_exp_breakdown_param_bias}
\end{equation}
which makes the dependence on the \emph{parameter bias} $\theta-\theta_0$ explicit.

The geometric decomposition in~\eqref{eq:goi_exp_breakdown} highlights \emph{prediction bias} at the output level through the discrepancy in predictive means. In contrast, the approximation
in~\eqref{eq:goi_exp_breakdown_param_bias} replaces the unobservable term
$A'(\eta_0)$ with the estimable \emph{parameter bias} $(\theta-\theta_0)$. This
re-expression attributes bias to concrete directions in parameter space via
$\nabla_\theta \eta_\theta(x)$ and recovers the familiar structure of GLMs, where
$b(x) \approx A''(\eta_\theta(x))\,x^\top(\theta-\theta_0)$.
This form will be leveraged in the subsequent comparison with predictive-entropy–based
acquisition.

\textbf{Connection to predictive entropy (PE) minimization.} The most widely used acquisition criteria in active learning are uncertainty-based.
In contrast, influence-maximizing acquisition does not explicitly require uncertainty
estimates of the predictive model $\mathcal{M}_\theta$. For neural networks, such
uncertainty estimates are often computationally expensive to obtain and unreliable in
the early stages of data acquisition, when the training set is small.

Despite this apparent difference, there is a close connection between influence maximization and predictive entropy minimization in Bayesian GLMs.
In particular, for GLMs with a canonical link, influence maximization can be
interpreted as an approximation to predictive entropy minimization, as formalized in
Proposition~\ref{prop:glm}.
\begin{prop} \label{prop:glm}
Influence maximization approximates predictive entropy minimization in
GLMs with a canonical link. Under mild assumptions, the two
objectives differ by an additional term that calibrates the directional parameter bias
up to a constant.
\end{prop} 
\begin{proof}
    Consider a generalized linear model in the form of (\ref{eq:exponential_family_main}) with a canonical link $\eta_\theta (x) = x^\top \theta$ and $T(y)=y$. 
    The conditional entropy of model after acquiring $(x_{\mathrm{c}},y_{\mathrm{c}})$ can be derived as 
    \begin{align}
        \mathcal H(\theta \mid \mathcal{D}\cup \{(x_{\mathrm{c}}, y_{\mathrm{c}})\})=\frac{1}{2}\,\mathbb E_{y_{\mathrm{c}}\sim p_0(\cdot\mid x_{\mathrm{c}})}\big[\log \vert H_{\theta_{\mathrm{c}}}^{-1}\vert\big]  + const
        \label{eq:conditional_entropy}
    \end{align} 
    where $H_{\theta_{\mathrm{c}}}$ is the Hessian of negative log likelihood with model parameter $\theta_c = \theta(\mathcal{D}\cup \{(x_{\mathrm{c}},y_{\mathrm{c}})\})$.
    
    Under the canonical-link assumption and properties (\ref{sec:exponential}.2) and (\ref{sec:exponential}.3), minimizing (\ref{eq:conditional_entropy}) can be simplified to minimize \citep{lewi2007efficient} 
    \begin{align}
       \mathbb E_{y_{\mathrm{c}}\sim p_0(\cdot\mid x_{\mathrm{c}})} \big[\log |H_{\theta_{\mathrm{c}}}^{-1}|\big] = \log |H_{\theta}^{-1}| -
       \log \Big(1+ x_{\mathrm{c}}^{\top}H_{\theta}^{-1}x_{\mathrm{c}}\,A^{\prime\prime}\big(\eta_\theta(x_{\mathrm{c}})\big) \Big). 
    \end{align}
    
    Denote $z=x_{\mathrm{c}}^{\top}H_{\theta}^{-1}x_{\mathrm{c}}\,A^{\prime\prime}\big(\eta_\theta(x_{\mathrm{c}})\big)$. Given that $A^{\prime\prime}\big(\eta_\theta(x_{\mathrm{c}})\big)$ is bounded and $x_{\mathrm{c}}^{\top}H_{\theta}^{-1}x_{\mathrm{c}}\rightarrow 0$ as $\vert\mathcal{D}\vert\rightarrow \infty $~\citep{Paninski2005asymptotic}, if $z$ is sufficiently small, the standard linear approximation $\log(1+z) = z+\mathrm{o}(z)$ further simplifies the acquisition objective function as 
    \begin{align}
        \mathsf{PE}(x_{\mathrm{c}})\,:=\, -\, x_{\mathrm{c}}^\top H_\theta^{-1}x_{\mathrm{c}}\,A^{\prime\prime}\big(\eta_\theta(x_{\mathrm{c}})\big)\label{eq:Ipe_glm}.
    \end{align}
    On the other hand, the acquisition score under GOIMDA (\ref{eq:goi_exp_breakdown_param_bias}) is 
    \begin{align}
         \mathsf{GOI}(x_{\mathrm{c}})\;:=\, \big[\nabla_{\theta}\,\mathcal{G}(\theta)\big]^\top  H_{{\theta}}^{-1} \, x_{\mathrm{c}} \;\underbrace{x_{\mathrm{c}} ^{\top}  ({\theta}-\theta_0)}_{\textbf{prediction bias}}\; A^{\prime\prime}\big(\eta_\theta(x_{\mathrm{c}} )\big),
        \label{eq:I_glm}
    \end{align}
    where the additional directional parameter bias term coincides with the prediction bias under the context of GLMs with a canonical link.
    
    By comparing objective in (\ref{eq:Ipe_glm}) with the one in (\ref{eq:I_glm}), 
    GOIMDA differs by multiplicative \emph{goal-alignment} and \emph{bias} weights:
    \begin{align}
        \mathsf{GOI}(x_{\mathrm c})\;\propto\;\mathsf{PE}(x_{\mathrm c})\cdot \frac{\langle H_\theta^{-1/2}\nabla_\theta\mathcal G(\theta),\; H_\theta^{-1/2}x_{\mathrm c}\rangle}{\|H_\theta^{-1/2}x_{\mathrm c}\|^2}\cdot \big[x_{\mathrm c}^\top(\theta-\theta_0)\big].
        \label{eq:goi_vs_pe}
    \end{align}
This observation indicates that GOIMDA implicitly calibrates model uncertainty. The only distinction is that GOIMDA weights each parameter dimension according to its current bias, whereas predictive entropy minimization treats all dimensions uniformly.
\end{proof}

Proposition~\ref{prop:glm} shows that, for canonical GLMs, GOIMDA inherits the same
$H_\theta^{-1}$ leverage structure
$A''(\eta_\theta(x_{\mathrm c}))\,x_{\mathrm c}^\top H_\theta^{-1} x_{\mathrm c}$
that underlies predictive-entropy selection. Crucially, GOIMDA \emph{calibrates} this
uncertainty signal through two additional factors:
(i) a \emph{goal-alignment} term that projects parameter updates onto the descent
direction of $\mathcal{G}$, and
(ii) a \emph{prediction-bias} term that emphasizes inputs where the model disagrees with
the data-generating mechanism.

As a result, GOIMDA is \emph{uncertainty-aware}; it preserves the predictive-entropy
signal, while actively steering acquisition toward directions that most effectively
reduce the stated goal and toward regions of systematic error, all without requiring a
posterior distribution over~$\theta$.

Consequently, although GOIMDA is derived as a goal-directed influence maximization
procedure and is therefore inherently exploitative, it remains \emph{exploration-aware}:
it preferentially scores candidates in high-leverage, high-variance regions, thereby
implicitly encouraging exploration. This should be distinguished from Bayesian
exploration, which integrates over a posterior on~$\theta$; here, $\theta(\mathcal D)$
is held fixed.

In active learning, where prediction calibration is rarely made explicit, the bias term
focuses acquisition on misfit or systematically mislabeled regions. Empirically, the
results in Section~\ref{sec:experiment_bias_term} show that incorporating this bias
achieves target accuracy with substantially fewer labels than omitting it, confirming
that \emph{calibrating bias materially improves sample efficiency}. In global
optimization, GOIMDA plays the role of an \emph{exploration-aware, exploitation-focused} alternative that mirrors the exploitation component of Bayesian optimization without
requiring a posterior.

\section{Example Applications of GOIMDA}
\label{sec:exponential_example} 
We instantiate the GOIMDA framework for three representative use cases: \emph{global optimization of noisy functions}, where the goal is to identify
the minimizer of a complex, expensive-to-evaluate objective under noise; 
\emph{hyperparameter tuning}, where the aim is to achieve optimal test performance while
minimizing the number of costly training-and-validation evaluations; and \emph{deep active learning}, where the objective
is to train a classification model that generalizes well using as few labeled samples as
possible.

Owing to the flexibility of the user-defined goal objective function, GOIMDA applies
naturally across a wide range of problem settings. As introduced in
Section~\ref{sec:method}, the goal objective can be defined to address global optimization
problems involving black-box objectives that lack analytical expressions and do not
admit first- or second-order derivatives; see~(\ref{eq:goal_opt}). We also presented the goal
functions tailored to deep active learning, where the objective is to optimize
performance on a validation or test set under a limited labeling budget; see~(\ref{eq:goal_nll}) and~(\ref{eq:goal_ent}).

In this section, we focus on exponential family models to illustrate how GOIMDA can be
instantiated in practice for both global optimization and active learning. For each
example, we first specify the corresponding goal objective $\mathcal{G}$, and then
derive the resulting influence-function–based acquisition rule
(\ref{eq:x_next}) used in Algorithm~\ref{algo:GOI_general}.

\subsection{Example Application I: Iterative Global Optimization with Noisy Observations}
We consider a global optimization problem of the form
$x_{\min} = \arg\min_x f(x)$ for an unknown objective function $f$. Direct evaluations of
$f$ are unavailable; instead, we observe noisy function values
$y = f(x) + \varepsilon$, where $\varepsilon$ is zero-mean noise. At each iteration, model $\mathcal{M}_\theta$ with parameters
$\theta := \theta(\mathcal{D})$ is trained on the currently acquired dataset
$\mathcal{D} = \{(u_i, v_i)\}_{i=1}^n$ to approximate the underlying function $f$.

We define the goal objective $\mathcal{G}$ as in (\ref{eq:goal_opt}), the expected value of the true function at the
model's recommended minimizer. Under the exponential family assumption with
$T(y)=y$, this objective can be written as  
\begin{align}
    \mathcal{G}(\theta) \;:=\;\mathbb E_{y\sim p_0(\cdot\mid \hat{x}^\ast_\theta)}[\, y\,] \;\stackrel{(\ref{sec:exponential}.2)}{=}\;  A^{\prime}\big(\eta_0(\hat x_\theta^\ast)\big),
    \label{eq:GO_noisy_goal}
\end{align} 
with the
model's recommended minimizer 
\begin{align}
    \hat{x}_\theta^\ast = \arg\min_{x} \mathbb{E}_{y\sim p_{\theta}(\cdot\mid x)}[\, y\,].
    \label{eq:GO_noisy_xmin}
\end{align}

By construction, GOIMDA selects the candidate whose acquisition yields
the largest \emph{instantaneous expected decrease} in the true objective evaluated at the
model's current recommendation. 

The gradient of the goal objective with respect to $\theta$ follows from the chain rule:
\begin{align}
    \nabla_\theta\mathcal{G}(\theta)\; := \;A^{\prime\prime}\big(\eta_0(\hat{x}_\theta^\ast)\big)\,\big[\frac{\partial}{\partial \theta}\hat{x}_\theta^\ast\big]\,\big[\nabla_{x}\eta_0(\hat{x}_\theta^\ast)\big].
    \label{eq:GO_goal_gradient}
\end{align}
The derivative $\tfrac{\partial}{\partial \theta}\hat{x}_\theta^\ast$ can be obtained via implicit differentiation of the optimality conditions defining
$\hat{x}_\theta^\ast$ (see Appendix \ref{sec:appendix_partial_x} for details), yielding the approximation
\begin{align}
    \frac{\partial}{\partial\theta}\hat x_\theta^\ast\;\approx\; \Big[\frac{\partial^2}{\partial \theta\partial x}  A^\prime\big(\eta_\theta(\hat x_\theta^\ast)\big)\Big]\,\big[\nabla_x^2 A^\prime\big(\eta_\theta(\hat x_\theta^\ast)\big)\big]^{-1}.
\end{align}

Substituting these expressions into the influence function
(\ref{eq:goi_exp_main}) yields the following explicit decomposition of the
Goal-Oriented Influence score:
\begin{equation} 
    \begin{aligned}
        \mathsf{GOI}(x_\text{c})\;=\;\underbrace{A''\!\big(\eta_0(\hat x_{\theta}^\ast)\big)\,\nabla_x \eta_0(\hat x_{\theta}^\ast)^\top}_{\substack{\textbf{how the true objective drops}\\\text{if }\hat x_{\theta}^\ast\text{ moves}}} 
        \underbrace{\big[\nabla_x^2 A^\prime(\eta_\theta(\hat x_{\theta}^\ast))\big]^{-1}\!\big[\tfrac{\partial^2}{\partial\theta\partial x}A^\prime(\eta_\theta(\hat x_{\theta}^\ast))\big]}_{\substack{\textbf{how }\hat x_{\theta}^\ast \textbf{ moves}\text{ when }\theta\text{ moves}}}
        \underbrace{H_\theta^{-1}\,\nabla_\theta\eta_\theta(x_{\text{c}})\,b(x_{\text{c}})}_{\substack{\textbf{how }\theta \textbf{ moves}\\\text{when adding }(x_{\mathrm c},y_{\mathrm c})}}, 
    \end{aligned} 
    \label{eq:goi_noisy_expand}
\end{equation}
where $b(x_\text{c}):=A^\prime\big(\eta_\theta(x_{\text{c}})\big) - A^\prime\big(\eta_0(x_{\text{c}}))\big)$ denotes the prediction-bias term introduced in
(\ref{eq:goi_exp_breakdown}).

In practice, the unknown true natural parameter $\eta_0(\cdot)$ is approximated using a
Jackknife resampling strategy with the deep ensemble model
$\tilde{\mathcal{M}}_\phi$, and inverse HVPs are computed
stochastically, as described in Section~\ref{sec:method}.

\subsection{Example Application II: Hyperparameter Optimization}
We next consider hyperparameter tuning for neural networks under a transfer learning
setup, which induces a more complex optimization problem. Hyperparameters determine the
model architecture and training configuration, and different choices of hyperparameters
typically lead to different trained parameters and, consequently, different test
performance. Evaluating even a single hyperparameter configuration requires training and
validating a neural network, which is computationally expensive. The goal is therefore to
identify hyperparameters that yield strong test performance while minimizing the number
of costly training and validation runs.

Formally, let $\xi\in \Xi$ denote the hyperparameters of a model trained on a fixed training
dataset
$\mathcal{D}^{\mathrm{tr}} = \{(x_i^{\mathrm{tr}}, y_i^{\mathrm{tr}})\}_{i=1}^n$, and let
$(x,y)\in \mathcal{X}\times \mathcal{Y}$ denote a test input–output pair, where $y$ may be unobserved.\footnote{For ease of
exposition, we consider a single test input $x$.} 
Denote the predictive test negative log-likelihood loss obtained by a model trained with hyperparameters $\xi$ as $h_0(x,\xi)$, which depends on both the test input $x$ and the hyperparameters $\xi$.
Due to stochastic effects such as random initialization and optimization noise, repeated training with the
same $\xi$ can yield different validation outcomes. As a result, we only observe noisy
evaluations of $h_0$ in its noisy form, $r=h_0(x,\xi)+\varepsilon$, where $\varepsilon$ denotes observation noise. The goal is to optimize $h_0$ with
respect to $\xi$ given $x$, as the ``mean'' test performance is indifferent to stochastic variability in the model training.

Let $q_0(\cdot \mid x,\xi)$ denote the true conditional distribution of the noisy
observation $r$ given the concatenated input $(x,\xi)$, parameterized by $\theta_0$. We
train a model with parameters $\theta$ on the currently acquired dataset
$\mathcal{D}$, whose elements take the form $([x,\xi], r)$. Under the exponential family
assumption, the likelihood is
\begin{align}
    q_\theta(r\mid x,\xi) \;=\; h(r)\exp\Big (\eta_\theta(x,\xi)\, T(r)-A\big(\eta_\theta(x,\xi)\big)\Big).
\end{align}

The search for the set of hyperparameters that yields optimal expected test performance is then formulated via
the goal objective
\begin{align}
    \mathcal{G}(\theta) \;:=\; \mathbb E_{r\sim q_{0}(\cdot\mid x,\hat{\xi}_{\theta}^\ast)}[\, r \,] \label{eq:hpo_formulation},
\end{align}
where $\hat{\xi}^{\ast}_{\theta}$ denotes the minimizer of model $\mathcal{M}_\theta$'s predictive test negative log-likelihood loss
\begin{align}
\hat{\xi}^{\ast}_{\theta}\;=\;
\arg\min_{\xi}\,
\mathbb{E}_{r \sim q_{\theta}(\cdot \mid x,\xi)}[\,r\,].
\label{eq:GO_hyperparam_xmin}
\end{align}
The formulation in (\ref{eq:hpo_formulation}) mirrors the iterative global optimization with noisy
observations in (\ref{eq:GO_noisy_goal}),
with the decision variable specialized to hyperparameters for fixed test inputs.
Consequently, the influence score for hyperparameter
acquisition can be derived analogously to~(\ref{eq:goi_noisy_expand}).

\subsection{Example Application III: Deep Active Learning}
We illustrate how Goal-Oriented Influence-Maximizing Data Acquisition (GOIMDA) can be
applied to deep active learning under exponential family distributions. We first define
a test-centric goal objective function $\mathcal{G}(\theta)$ and then derive a closed-form
\emph{influence score} for adding a candidate $(x_{\mathrm c}, y_{\mathrm c})$. The
resulting acquisition criterion couples the $H_\theta^{-1}$ geometry with a
\emph{prediction-bias} weighting.

Formally, consider a classification task where $x \in \mathbb{R}^d$ denotes the input
features and $y \in \mathbb{Z}$ denotes the output label. This setting also captures the
motivating biological example, where $x$ corresponds to experimental stimuli and $y$ to
the measured system response, and where conducting experiments to read the system response is costly. At each iteration, a
deep neural network $\mathcal{M}_\theta$ with parameters
$\theta := \theta(\mathcal{D})$ is trained on the currently labeled dataset
$\mathcal{D}$ to model the conditional distribution of the response $y$ given the stimuli $x$. Test inputs
$(x,y)$ are drawn from the same unknown data-generating distribution $p_0$, with $y$
unobserved at acquisition time.

The objective of active learning is to iteratively select inputs $x_{\mathrm{next}}$ for
label acquisition such that the model achieves low test loss
$\ell(\theta; (x, y))$ using as few labeled samples as possible. 

To optimize the model performance on test data $(x,y)$ where $y$ is unknown, a natural test-centric
goal objective is
\begin{align}
     \mathcal{G}(\theta)\; :=\; \mathbb E_{y\sim p_0(\cdot\mid x)}\big[\ell \big(\theta; (x,y)\big)\big].
     \label{eq:al_example_goal}
\end{align} 
Using the explicit exponential-family loss in~(\ref{eq:exp_loss_main}), the gradient of
$\mathcal{G}$ takes the form
\begin{align}
    \nabla_\theta \mathcal{G}(\theta)\;=\;\big[A^\prime\big(\eta_\theta(x)\big) - A^\prime\big(\eta_0(x))\big)\big]\,\nabla_\theta\eta_\theta(x).
\end{align}

Substituting this expression into the influence function
(\ref{eq:goi_exp_main}) yields the influence score for a candidate $x_{\mathrm c}$:
\begin{align}
    \mathsf{GOI}(x_{\text{c}})&\;=\;\big[A^\prime\big(\eta_\theta(x)\big) - A^\prime\big(\eta_0(x))\big)\big]\,\nabla_\theta\eta_\theta(x)^\top H_{\theta}^{-1}   \nabla_\theta\eta_\theta(x_{\text{c}})\,\big[A^\prime\big(\eta_\theta(x_{\text{c}})\big) - A^\prime\big(\eta_0(x_{\text{c}}))\big)\big]\nonumber \\
    &\;\propto\; \nabla_\theta\eta_\theta(x)^\top H_{\theta}^{-1}   \nabla_\theta\eta_\theta(x_{\text{c}})\,\big[A^\prime\big(\eta_\theta(x_{\text{c}})\big) - A^\prime\big(\eta_0(x_{\text{c}}))\big)\big],
    \label{eq:al_influence_fn_v0}
\end{align}
where the proportionality follows by discarding terms independent of $x_{\mathrm c}$.
Consequently, the next point to acquire is
\begin{align}
    x_{\text{next}} \;=\; \arg\max_{x_{\text{c}}}\,\nabla_\theta\eta_\theta(x)^\top H_{\theta}^{-1}   \nabla_\theta\eta_\theta(x_{\text{c}})\,\big[A^\prime\big(\eta_\theta(x_{\text{c}})\big) - A^\prime\big(\eta_0(x_{\text{c}}))\big)\big].
    \label{eq:al_influence_fn_v0_simplified}
\end{align}

Under the canonical-link assumption with $T(y)=y$, this expression simplifies to
\begin{align}
    x_{\text{next}} \;=\; \arg\max_{x_{\text{c}}}\, x^\top  H_{{\theta}}^{-1} \, x_{\mathrm{c}} \;\underbrace{x_{\mathrm{c}} ^{\top}  ({\theta}-\theta_0)}_{\textbf{prediction bias}}\; A^{\prime\prime}\big(\eta_\theta(x_{\mathrm{c}} )\big),
\end{align}
which recovers~(\ref{eq:I_glm}) with the goal objective defined
in~(\ref{eq:al_example_goal}) for the active learning setting.

In practice, the prediction-bias term is estimated using a Jackknife resampling strategy
via the deep ensemble surrogate $\tilde{\mathcal{M}}_\phi$, and inverse
HVPs are computed stochastically (see
Section~\ref{sec:method}).

\section{Empirical Studies}
\label{sec:emipirical_studies}
We empirically evaluate GOIMDA on a diverse set of controlled and realistic learning and
optimization tasks.\footnote{Code to reproduce our experiments is available at
\href{https://github.com/weichiyao/GOIMDA}{github.com/weichiyao/GOIMDA}.}
We begin with a synthetic logistic-regression study designed to isolate the effect of
the parameter-bias term in the acquisition rule. Next, we evaluate GOIMDA
for noisy global optimization of black-box test functions, benchmarking against standard
Gaussian-process Bayesian optimization methods with common acquisition functions
\citep{wang2017MES,scott2011KGgaussianmodel,jones1998EI,Torn1989PI,Srinivas2010UCB}.
We also apply GOIMDA to hyperparameter tuning on CIFAR-10 under distribution shift,
where the goal is to select hyperparameters that improve performance on a target test
distribution using labels only from a different source distribution. Finally, we study predictive learning on
image and text classification benchmarks \citep{mnist,cohen2017emnist,text_movie},
comparing GOIMDA against uncertainty-based active learning baselines such as BALD
\citep{Houlsby2011BayesianAL,gal2017deep,Kirsch2019batchBALD}.

\subsection{On the importance of the parameter-bias term in~(\ref{eq:goi_exp_breakdown_param_bias})} \label{sec:experiment_bias_term}

The approximation in~(\ref{eq:goi_exp_breakdown_param_bias}) makes the dependence on the
parameter bias $\theta-\theta_0$ explicit. This parameter-bias form is useful in two
ways. Conceptually, it attributes prediction bias to specific parameter directions via
$\nabla_\theta\eta_\theta(x)$, which supports diagnostics and motivates targeted
regularization. Computationally, it enables a plug-in implementation using
Jackknife/Bootstrap surrogates of $\theta_0$, requiring a single solve for
$H_\theta^{-1}\nabla_\theta\mathcal{G}$ and per-candidate vector--Jacobian (or
Jacobian--vector) products. Hence, the parameter-bias form is both interpretable and computationally practical.

\begin{wrapfigure}{r}{0.5\textwidth}
    \centering
    \includegraphics[scale=0.6]{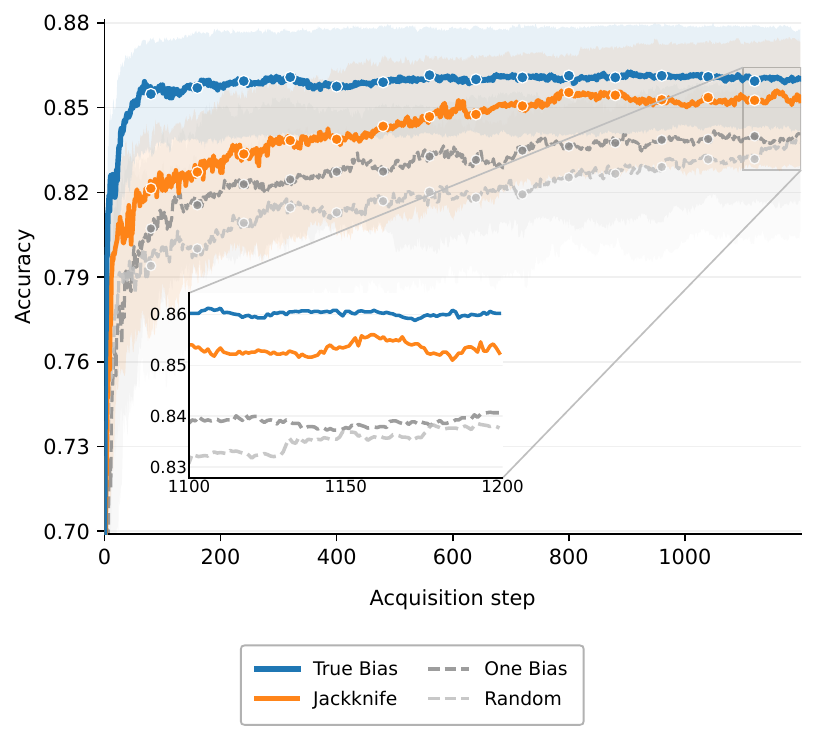}
    \caption{ \small  
     \textbf{The parameter-bias term is crucial for effective data acquisition, and a good approximation significantly improves the acquisition performance.}
     Using the true bias term (blue)
     reaches high accuracy with fewer acquisitions, and a principled approximation (orange)
     outperforms ignoring the bias term (dark gray dashed).
     The inset zoom highlights the persistent gap. Shaded regions indicate variability across runs.
     Test accuracy is shown versus the number of acquired labels (averaged over 200 repetitions). Higher is better.} 
    \label{fig:sim_logistic}
\end{wrapfigure}

In canonical GLMs (Proposition~\ref{prop:glm}), GOIMDA shares the same $H_\theta^{-1}$
leverage structure as predictive-entropy selection; the key difference is GOIMDA's bias
calibration, which re-weights parameter directions by
their current bias instead of treating them uniformly. In particular, apart from goal alignment, the additional factor is the prediction-bias
term $x_{\mathrm c}^\top(\theta-\theta_0)$. If $\theta-\theta_0$ is replaced
by a vector of ones, the influence objective~(\ref{eq:I_glm}) reduces to the
predictive-entropy objective~(\ref{eq:Ipe_glm}) up to a scaling constant, recovering a
bias-agnostic leverage rule.

To quantify the practical importance of $\theta-\theta_0$, we run a controlled
simulation that compares acquisition variants based on~(\ref{eq:I_glm}). Inputs
$x\in\mathbb{R}^d$ are generated from a low-rank latent-variable model
$z\in\mathbb{R}^l$, $l<d$ \citep{PPCA}: $x=Wz+\epsilon$, with
$z\sim\mathcal{N}(0,I)$ and $\epsilon\sim\mathcal{N}(0,\sigma^2I)$. We set $d=20$,
$l=3$, and $\sigma=0.1$. Outcomes are binary with $p_0(y\mid x)$ given by a Bernoulli
GLM. For each of 200 repetitions, we generate $50{,}000$ training points and
$5{,}000$ test points. At each acquisition step, we refit logistic regression on the
labeled set and select the next point by maximizing~(\ref{eq:I_glm}) under different
treatments of the bias term; random pool sampling serves as a baseline. All runs start
from two labeled points (one per class), and the remaining training points form the candidate pool. Test accuracy is recorded after each acquisition.

Figure~\ref{fig:sim_logistic} compares four acquisition strategies including (i) \textit{True Bias}: influence maximization using the true bias term $\theta-\theta_0$; (ii) \textit{Jackknife}: influence maximization using a Jackknife estimate of $\theta-\theta_0$; (iii) \textit{One Bias}: influence maximization with $\theta-\theta_0$ replaced by a vector of ones; and (iv) \textit{Random}: random acquisition from the pool.

As expected, \textit{True Bias} reaches high accuracy with the fewest acquisitions, and
\textit{Jackknife} closely tracks it. \textit{One Bias} still improves over random
sampling but consistently lags behind the bias-aware variants. Overall, ignoring the
bias term reduces acquisition efficiency: the better we approximate $\theta-\theta_0$,
the faster accuracy improves, and the fewer labels are required.

\subsection{Noisy black-box function optimization}

We benchmark GOIMDA on noisy Branin, Drop-Wave, and Ackley functions, the black-box test functions from \citep{bingham2015testfunc},
where each query returns $y=f(x)+\varepsilon$ with
$\varepsilon\sim\mathcal{N}(0,\sigma^2)$. We consider $\sigma\in\{0.1,0.2\}$ and start
each run with 5 initial observations.
See full details in Appendix \ref{appendix:noisy_black_box}.

Figure \ref{fig:ao_noisy_function} provides a performance comparison on 2D Branin and 5D Ackley functions between GOIMDA and five other commonly used Gaussian Processes (GP)-based Bayesian optimization methods with different acquisition functions.
In each of these plots, we show the \textit{immediate regret}, which measures the difference between the outcome of the best possible decision ($\min_{x\in\mathcal{S}} f(x)$ for some feasible set $\mathcal{S}$) and the outcome of the decision made by each active optimization method. Given its noisy nature, the regret value may go up in consecutive acquisition steps, but should, in general, decline if solved by an effective optimization algorithm.

Across all three benchmarks, GOIMDA (blue dashed) reduces immediate regret much faster than the GP baselines, typically reaching the lowest regret within the first few dozen acquisitions and then stabilizing. On Branin, GOIMDA quickly drops below the other methods for both noise levels; when noise increases from $\sigma^2=0.01$ to $\sigma^2=0.04$, the GP baselines plateau at noticeably higher regret while GOIMDA maintains a substantially lower plateau, widening the gap. On Ackley, the contrast is even sharper: GOIMDA rapidly drives regret down by orders of magnitude, whereas other GP baselines improve slowly and remain far above GOIMDA throughout; this separation persists and effectively becomes more pronounced under the higher noise setting. On Dropwave, all methods are more competitive, but GOIMDA still achieves the lowest or near-lowest regret earlier, and the advantage becomes clearer at $\sigma^2=0.04$, where increased noise slows the baselines more than GOIMDA. Overall, higher noise degrades all methods, but GOIMDA is markedly more noise-robust, retaining rapid early gains and a lower final regret.

\subsection{Hyperparameter tuning under distribution shift}

We study hyperparameter optimization when training labels are available only from a labeled source distribution $(X,Y)$, but performance is evaluated on an unlabeled target distribution under distribution shift $(\tilde{X},\tilde{Y})$ with unknown $\tilde{Y}$. 
The marginals $p_X$ and $p_{\tilde{X}}$ may differ, while the conditional model is assumed to share the same form $p_0:=p_{\theta_0}(Y\mid X)$ for unknown $\theta_0$.

Concretely, we instantiate this setting on CIFAR-10 with a Pre-Activation ResNet backbone. We construct an imbalanced labeled source set and an unlabeled target set drawn from a restricted set of classes. 
Full details of the dataset construction and splits can be found in Appendix \ref{appendix:hyperparameter_tuning}.

\begin{wrapfigure}{r}{0.5\textwidth}
    \centering
    \includegraphics[scale=0.63]{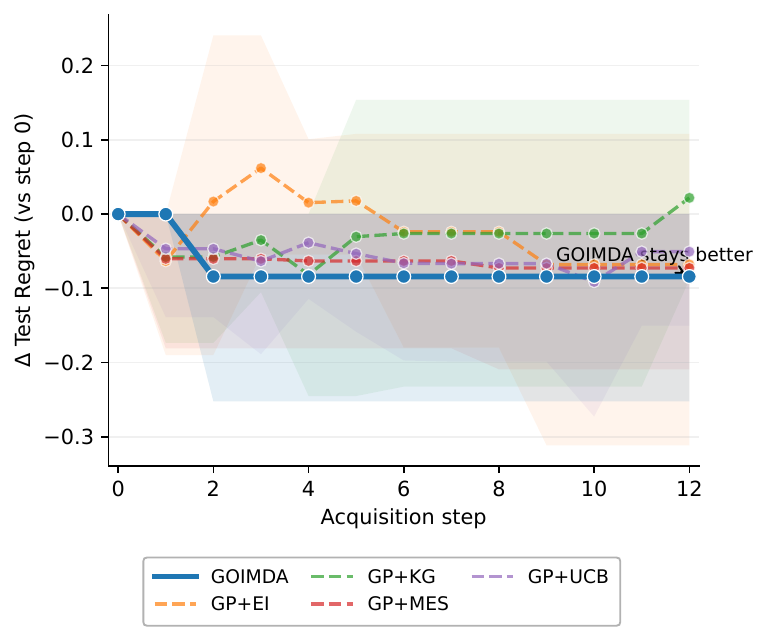}
    \caption{ {\small
    \textbf{GOIMDA selects the hyperparameters that provide the best prediction performance on the target test set compared to Bayesian optimization methods}.
    The performance is evaluated in terms of the $\Delta$ test regret relative to the initial configuration (step 0) across acquisition steps. 
    Solid line shows GOIMDA; dashed lines show Bayesian optimization baselines. Shaded regions indicate 95\% bootstrap confidence intervals of the mean across 3 trials. Negative values correspond to improved target-set performance.
    }
    }
    \label{fig:HPO_CIFAR10}
\end{wrapfigure}

In this setting, hyperparameter selection is driven only by source training/validation
performance can favor configurations that fit the imbalanced source distribution, but
transfer poorly to the target distribution. GOIMDA instead evaluates hyperparameters by
their estimated influence on the expected prediction loss on the target distribution,
while still leveraging the labeled source data for training. 

Figure~\ref{fig:HPO_CIFAR10} shows that GOIMDA consistently identifies hyperparameter settings that improve target performance under distribution shift more reliably than Bayesian optimization baselines. Measured by the $\Delta$ test regret relative to the initial configuration (step 0), GOIMDA quickly attains a stable negative improvement after only a few acquisitions and maintains the best mean regret across subsequent steps. In contrast, the Bayesian optimization baselines exhibit larger variance and less consistent progress, often oscillating around smaller improvements and occasionally reverting toward worse target performance; this indicates that optimizing based on source-driven signals can select configurations that do not transfer well. Overall, the results suggest that explicitly scoring candidates by their estimated influence on target loss makes hyperparameter search more sample-efficient and robust to distribution shift.

\subsection{Predictive learning} 

We evaluate GOIMDA in standard active learning loops: at each iteration, the model is
trained on the current labeled set and then selects the next point from a pool
$\mathcal{S}$. Across datasets, we use comparable feedforward architectures (stacked
\textsf{linear-relu} layers followed by \textsf{softmax}) for all methods. BALD adds
\textsf{dropout} layers and uses MC dropout for approximate inference \citep{gal2017deep}. All
models are trained with Adam \citep{adam2014} (learning rate $0.001$, $\beta_1=0.9$,
$\beta_2=0.999$) on a single GPU. Test accuracy is measured after each acquisition.
Results are aggregated over four independent trials; error bars show the mean and the
lower/upper quartiles.

We consider three benchmarks spanning image and text classification: MNIST \citep{mnist}, EMNIST Letters \citep{cohen2017emnist}, and binary sentiment classification derived from the Rotten Tomatoes phrase dataset \citep{movie2005}. For each benchmark, we follow a standard pool-based protocol with a small balanced initialization, a held-out validation set, and the remaining training points as the acquisition pool; see Appendix~\ref{appendix:predictive_learning} for exact split sizes and preprocessing choices.

\begin{table}[t]
  \caption{\textbf{GOIMDA requires fewer samples than BALD and random acquisition to achieve high accuracy on MNIST}. The table reports 25\%-, 50\%-, and 75\%-percentiles for the number of required data points to reach 80\%, 90\%, and 95\% accuracy on MNIST. Higher is better.}
  \vspace{1em}
  \label{tbl:mnist}
  \centering
  \begin{tabular}{rC{3cm}C{3cm}C{3cm}}
    \toprule
    & \multicolumn{3}{c}{Acquisition Methods}  \\
    \cmidrule{2-4}
    \% Accuracy &  GOIMDA   & BALD     & Random \\
    \midrule
    80\% &  $88/108/116$    & $104/109/117$  &  $138/150/173$   \\
    90\% &  $374/422/441$   &   $432/445/475$ &     $846/901/977$ \\
    95\% &  $1060/1150/1206$ &   $1514/1589/1678$      & $4265/4522/4731$   \\
    \bottomrule
  \end{tabular}
\end{table}

\paragraph{MNIST} 
We follow the standard protocol of \citep{gal2017deep} with a small balanced initial labeled set and a fixed validation hold-out; remaining points form the pool (see Appendix \ref{appendix:mnist} for more details). 
GOIMDA and Random acquisitions are assessed with an architecture of three \textsf{linear-relu} layers with hidden widths 392 and 128. BALD uses a
similar architecture with two \textsf{linear-dropout-relu} layers followed by the final {linear-softmax} output layer, with the same hidden unit dimensions. We use $100$ MC dropout samples in BALD.
The expected loss for the test data points and that for candidate data points in GOIMDA are approximated using 10 Jackknife samples.

\paragraph{EMNIST Letters} 
We follow a similar pool/validation setup as in MNIST (see Appendix \ref{appendix:emnist} for more details).
Similar to what is used in MNIST, GOIMDA, and Random have three \textsf{linear-relu} layers with hidden unit dimensions $392$ and $128$, respectively, and 
BALD adds two additional \textsf{dropout} layers for inference. In addition, $50$ MC dropout samples are used in BALD for acquisition.  
To compute the expected loss for the test data points and the expected loss for candidate data points, $5$ Jackknife samples are used to approximate the expected loss in GOIMDA.

\paragraph{Movie reviews (Rotten Tomatoes)} 
For the Rotten Tomatoes dataset, we form a binary task by mapping the original 5-level sentiment labels into $\{0,1\}$ after removing neutral phrases, then apply a bag-of-words preprocessing pipeline; see full details in Appendix \ref{appendix:movie}.
The initial training set contains a balanced dataset of size 20. 
GOIMDA and Random have three \textsf{linear-relu} layers with hidden unit dimensions $64$ and $64$, respectively. BALD again adds two additional \textsf{dropout} layers for inference. In addition, $50$ MC dropout samples are used in BALD for acquisition.  
To compute the expected loss for the test data points and the expected loss for candidate data points, $10$ Jackknife samples are used to approximate the expected loss in GOIMDA.

\begin{figure}[t]
    \centering
    \includegraphics[scale=0.435]{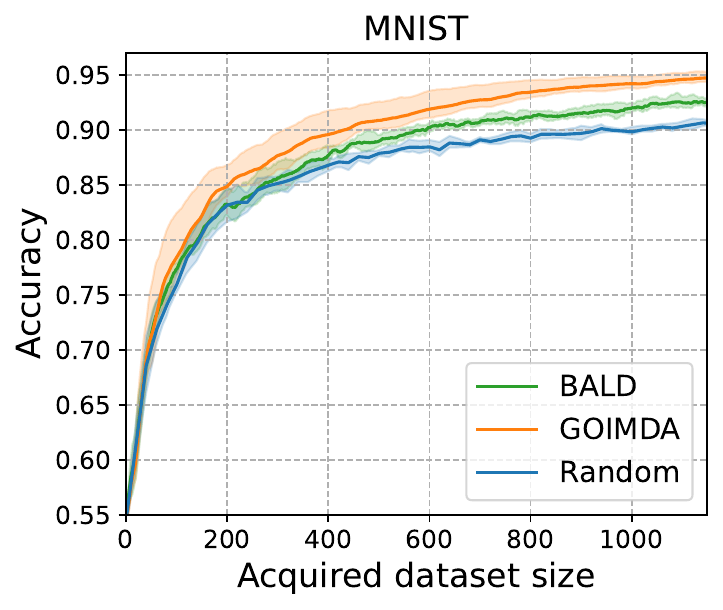}
    \includegraphics[scale=0.435]{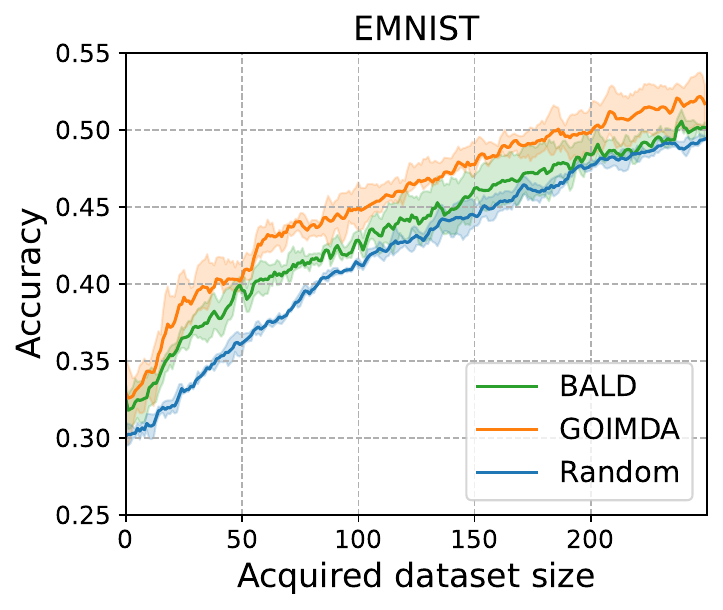}
    \includegraphics[scale=0.435]{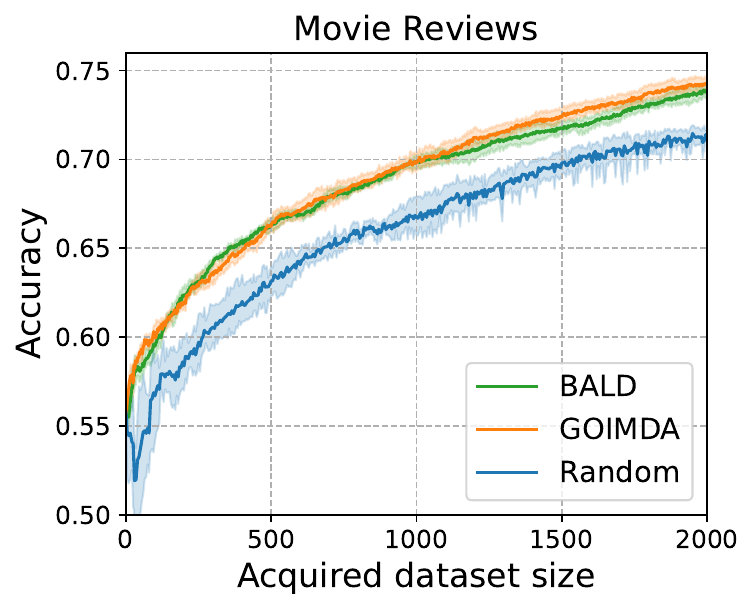}
    \caption{ {\small
    \textbf{GOIMDA outperforms both random acquisition and BALD on classification tasks on images of digits from MNIST (Left) and EMNIST (Middle), and the sentiment of movie reviews from the Rotten Tomatoes dataset (Right)}. The test accuracy is evaluated at each acquisition step.
    Left: GOIMDA outperforms both random acquisition and BALD on MNIST. 
    Middle: GOIMDA outperforms both random acquisition and BALD on EMNIST, whereas BALD only performs slightly better than random acquisition.
    Right: Both GOIMDA and BALD outperform the random sampling in terms of the accuracy rate at a given acquired dataset size. While GOIMDA and BALD have similar performance, GOIMDA gives slightly higher accuracy rates after the total acquisition size reaches $1,000$. 
    }}
    \label{fig:active_learning}
\end{figure}

Table~\ref{tbl:mnist} and Figure~\ref{fig:active_learning} show that GOIMDA is consistently more sample-efficient than BALD and random acquisition. On MNIST, GOIMDA reaches high accuracy with much fewer labeled examples---often requiring only a small fraction of the labeling budget needed by random sampling---highlighting the value of targeted acquisition. 
The learning curves in Figure~\ref{fig:active_learning} reinforce this trend, with GOIMDA maintaining the highest test accuracy throughout the acquisition process. 
On EMNIST Letters, GOIMDA again achieves the best performance, delivering a clear accuracy advantage over both BALD and random, while BALD provides only marginal improvement over random acquisition. 
On Movie Reviews, GOIMDA and BALD both outperform random sampling at essentially every acquisition step, with GOIMDA matching or slightly exceeding BALD; the tighter uncertainty bands for GOIMDA/BALD also suggest more stable gains than random selection.

\section{Related work}
\label{sec:related}

This work draws on three themes around iterative
data acquisition under expensive evaluation or labeling costs.

\subsection{Iterative data acquisition for optimization}

A standard approach to solving expensive black-box optimization
problems is iterative data acquisition, in which a surrogate model is often
fit to observed evaluations, and an acquisition function is used to
select new inputs~\citep{Booker1999surrogateclass,Regis2007RBFGO,REGIS2007ParallelRBFGO,Garnett2023bayesoptbook}.
The most influential instantiation of this paradigm is \emph{Bayesian
optimization} (BO) \citep{Shahriari2016BOreview,Frazier2018ATO,Wang2023review},
which maintains a probabilistic surrogate $\mathcal M$ of the unknown objective and selects new evaluations by optimizing an acquisition function that balances exploration and exploitation.

Classical BO methods typically rely on Gaussian process (GP) surrogates
\citep{rasmussen2006GPforML} together with acquisition criteria such as
Probability of Improvement (PI) \citep{Torn1989PI,Kushner1964PI}, Expected Improvement (EI) \citep{jones1998EI,Mockus1975EI}, Upper Confidence Bound (UCB) \citep{Srinivas2010UCB}, Entropy Search (ES) \citep{hennig2012entropysearch,Frolich2020NoisyInput_EntropySearch}, Max-value Entropy Search (MES) \citep{wang2017MES}, and Knowledge Gradient (KG) \citep{frazier2008KLdiscrete,frazier2009KLcorrelatednormal,scott2011KGgaussianmodel,Wu2016KG,Wu2019PracticalMB}. 

Despite its empirical success, BO exhibits several practical limitations \citep{Shahriari2016BOreview,Frazier2018ATO,Garnett2023bayesoptbook,Wang2023review}.
First, GP-based surrogates require forming and factorizing an $n \times n$ kernel matrix at each iteration, incurring $\mathcal{O}(n^2)$ memory and $\mathcal{O}(n^3)$ time complexity \citep{rasmussen2006GPforML,Garnett2023bayesoptbook}. This rapidly becomes prohibitive as the number of evaluations grows. To address scalability or high-dimensional settings, prior work has proposed structured or high-dimensional GP variants, low-dimensional embeddings, and localized trust-region methods such as TuRBO \citep{wang2018batchBOhighdim,Oh2018BOCylindricalKernels,Nayebi2019BOEmbeddedSubspaces,Shen2023HDBO,Eriksson2019ScalableLocalBO,Eriksson2021ScalableCB}. While effective in some regimes, these approaches often impose restrictive modeling assumptions, introduce additional posterior-maintenance overhead, and can be sensitive to model miscalibration \citep{Binois2022HDGPBO,Wang2023review}.
Second, replacing GPs with more scalable surrogates—such as tree- or density-based models \citep{Bergstra2011treebasedBO,Hutter2011SMAC} or Bayesian neural networks \citep{Snoek2015scalableBODNN,Springenberg2016BORobustBNN}—can alleviate computational costs, but frequently at the expense of uncertainty calibration or robustness \citep{Binois2022HDGPBO,li2024BNNBOreview}.
Third, many acquisition functions require expectations under the surrogate’s posterior predictive distribution. Fully Bayesian treatments, involving hyperparameter marginalization and exact acquisition computation, are rarely tractable in practice. Consequently, BO typically relies on plug-in hyperparameter estimates \citep{jones1998EI,Eriksson2019ScalableLocalBO} or Monte Carlo approximations \citep{Springenberg2016BORobustBNN,Snoek2015scalableBODNN,Wu2016KG,Wu2019PracticalMB,kim2022deeplearningbayesianoptimization}, which respectively forgo proper marginalization and introduce additional computational overhead and estimator variance \citep{brochu2010tutorialbayesianoptimizationexpensive,Shahriari2016BOreview,Binois2022HDGPBO,li2024BNNBOreview,Garnett2023bayesoptbook}.

Motivated by these challenges, we pursue an alternative that avoids explicit posterior inference altogether. The proposed GOIMDA algorithm scores candidate points by the first-order effect that upweighting them would have on a task-level objective. This yields an exploration-aware yet bias-focused improvement strategy. Under exponential-family models, GOIMDA retains the inverse curvature and variance terms that underlie predictive-entropy-based acquisitions, thereby capturing uncertainty while modulating them through goal alignment and predictive bias to emphasize exploitation. As we show in subsequent sections, GOIMDA recovers the exploitation behavior of Bayesian optimization without maintaining a posterior, offering a practical alternative when Bayesian updates are computationally expensive or poorly calibrated \citep{Shahriari2016BOreview,Wang2023review,Garnett2023bayesoptbook}.

\subsection{Iterative data acquisition for learning}

A closely related paradigm in supervised learning is \emph{active
learning} (AL), which aims to construct accurate predictive models
using as few labeled examples as possible \citep{Settles2009ActiveLL}.
AL addresses supervised learning settings in which labeled data are expensive to obtain and the learner must adaptively decide which inputs to query. A canonical example arises in biological stimulus–response studies, where experimental cost limits the number of probes and careful selection of stimuli is crucial for learning accurate predictive models.

Here, we focus on AL for deep learning (DL) models, which have achieved remarkable success across domains due to their expressive representations and strong function-approximation capabilities. Most deep AL methods have been developed for classification tasks, particularly in image and text domains \citep{ren2021surveyDAL,Schrder2020surveyALforText,Li2025surveyDALrecent}, and are also covered in broader surveys of active learning \citep{Settles2009ActiveLL}.

In pool-based deep AL, an \emph{acquisition function} assigns a utility score to each unlabeled data point, and the learner queries those with the highest scores at each iteration. Existing acquisition strategies for deep neural networks can be broadly categorized into three families \citep{ren2021surveyDAL,Li2025surveyDALrecent}:
(i) \emph{uncertainty-based} methods, which query points about which the model is most uncertain \citep{Houlsby2011Bald,Wang2014LC,Bloodgood2018margin,Li2020margin,wang2016ALImage,Ranganathan2017,gal2017deep,yang2017suggestive,Lakshminarayanan2017ALpredictiveundertainty,beluch2018power,siddhant2018deep,Tran2019BayesianGA,Kim2021LADA,BatchBAL2023,Kirsch2019batchBALD,Xie2022energy};
(ii) \emph{representation-based} methods, which select points that are prototypical or diverse in a learned feature space \citep{Chakraborty2015repAL,sener2018coreset,Zhao2019repAL,Li2020repAL,Hasan2020repAL,Coleman2022repAL,Gudovskiy2020repAL,kim2022coreset,Jin2022repAL,LiS2022repAL,Parvaneh2022repAL}; and
(iii) \emph{hybrid} methods, which combine uncertainty with diversity or representativeness in feature or gradient space \citep{Donmez2007hybridAL,Yin2017hybridAL,Zhdanov2019hybridAL,Sinha2019VariationalAA,Ash2020DBALgradient,Shui2020hybrid,Citovsky2021hybridAL,Gu2021hybridAL,Huang2021hybridAL,Geifman2019hybridAL,Ash2021fishing,saran2023streaming}.
The latter two categories are particularly common in batch-mode acquisition, where the goal is to select a set of informative yet non-redundant points in each AL round.

Among these approaches, \emph{uncertainty-based} criteria remain the most widely used in both classical and deep active learning \citep{Settles2009ActiveLL,Schrder2020surveyALforText}. These methods are typically instantiated via confidence-, margin-, or entropy-based scores \citep{Wang2014LC,Schohn2000margin,Roth2006margin,Bloodgood2018margin,Li2020margin}, Bayesian mutual-information objectives such as BALD and its batch extensions \citep{Houlsby2011Bald,gal2017deep,BatchBAL2023,Kirsch2019batchBALD}, or expected model-change criteria that approximate loss reduction \citep{Settles2007MultipleInstanceAL,Roy2001modelchange,Freytag2001modelchange,Tran2019BayesianGA,xu2019understandinggoalorientedactivelearning}. 
A recent work proposed to score candidate queries by how much they are expected to reduce predictive uncertainty on target inputs, making the acquisition objective directly aligned with downstream predictive performance~\citep{bickfordsmith2023epig}. 
In deep learning, these strategies commonly rely on approximate Bayesian neural networks (e.g., Monte Carlo dropout) \citep{Srivastava2014dropout,Gal2016Dropout,gal2017deep} or deep ensembles \citep{Lakshminarayanan2017ALpredictiveundertainty,beluch2018power}, and have also been combined with discriminative or energy-based scoring functions \citep{beluch2018power,Kim2021LADA,Xie2022energy}.

Despite their empirical success, uncertainty-based deep AL methods face fundamental challenges. Predictive probabilities from modern neural networks are often poorly calibrated, with models exhibiting overconfidence even when incorrect or under distribution shift \citep{guo2017calibration,abdar2021review,Schrder2020surveyALforText}. Moreover, Bayesian neural networks require approximate inference procedures that are computationally expensive and can be unreliable in the small-data regimes typical of early AL rounds \citep{arbel2023primerBNN}. As highlighted in recent surveys, obtaining reliable and well-calibrated uncertainty estimates in deep models remains an open problem, which directly limits the robustness of uncertainty-driven deep active learning \citep{Schrder2020surveyALforText,ren2021surveyDAL,Li2025surveyDALrecent}.

In contrast, GOIMDA avoids explicit posterior inference altogether while remaining uncertainty-aware. We build on influence functions as an alternative basis for data acquisition. Rather than explicitly modeling predictive uncertainty, influence-based methods estimate how upweighting a candidate point would affect the trained model or a downstream evaluation objective, without retraining from scratch \citep{koh2017understanding}. By using influence scores for query selection, GOIMDA avoids reliance on calibrated uncertainty estimates while still implicitly capturing aspects of uncertainty and prediction bias. As we show in Section~\ref{sec:exponential}, under an exponential-family assumption, GOIMDA admits a connection to predictive entropy minimization: it retains uncertainty-sensitive inverse curvature terms while explicitly incorporating a prediction-bias component, resulting in an exploration-aware yet exploitation-focused acquisition rule.

\subsection{Influence-function-based algorithms}

Influence functions (IFs) originate in robust statistics as a principled tool for quantifying the infinitesimal effect of perturbing a data point on an estimator. Seminal work by Cook and colleagues developed a comprehensive diagnostic framework for identifying influential observations and assessing local influence in regression models, encompassing single-point and grouped perturbations, residual-based diagnostics, and empirical characterizations of influence \citep{cook1977detection,cook1980characterizations,cook1986assessment}.

These ideas were later revived in the deep-learning literature to analyze complex, nonconvex models. In particular, \citep{koh2017understanding} demonstrated that first-order IFs can approximate the effect of upweighting or removing a training example on a model’s predictions, enabling practical tools for debugging, dataset curation, and attributing test errors to individual training points. Subsequent work extended IFs to group-level effects \citep{koh2019accuracy}, scalable approximations \citep{guo2020fastif}, higher-order and group influence \citep{basu2020influence,basu2020second}, and uncertainty quantification via higher-order IF and jackknife-style estimators \citep{alaa2020discriminative}. More broadly, IF-based analyses have been used to uncover spurious data artifacts \citep{han2020explaining}, characterize memorization and long-tail behavior \citep{feldman2020neural}, and study representational bias in neural networks \citep{brunet2019understanding}.

Most existing applications of IFs in deep learning are retrospective: given a fixed, labeled training set—and often fixed test points—IFs are used to explain or diagnose existing predictions. In contrast, iterative data acquisition presents a fundamentally different setting. At selection time, the label of a candidate point is \emph{unknown}, and acquisition must therefore be based on the \emph{expected} influence of the candidate under a predictive distribution $p_0(y \mid x)$.

Prior work that leverages influence functions for iterative data acquisition has largely focused on active learning. An initial line of work proposes selecting data points based on their influence with respect to predefined utility functions \citep{xu2019understandinggoalorientedactivelearning,Liu2021InfluenceSelection,Xia2023RALIF}. These utility functions encode the learning ``goal,'' such as the training loss on an auxiliary labeled test set or predictive entropy computed from the current model. When test labels are unavailable, or when the candidate’s label is unknown, these methods rely on approximations using the current model $p_\theta$, in the same spirit as standard uncertainty-based active learning. A related approach \citep{wang2022boostingactivelearningimproving} avoids directly evaluating test loss by deriving an upper bound and selecting points based on the gradient norm of the candidate loss. However, terms involving unknown candidate outputs are again approximated using the current trained model alone.

Different from these works, we study the influence of a candidate data point on a \emph{goal function} that is flexibly defined to cover a broad range of tasks. These include, for example, iterative optimization of black-box functions and label acquisition to improve performance on an unlabeled test set in active learning. Rather than approximating unknown quantities solely through the current trained model, we estimate them using a separate ensemble constructed via resampling techniques such as the jackknife. This decouples acquisition from a single potentially miscalibrated predictor and enables influence-based selection that remains effective when posterior uncertainty is expensive or unreliable to estimate.

\section{Discussion}
\label{sec:discussion}
In this work, we introduce \emph{Goal-Oriented Influence-Based Data
Acquisition} (GOIMDA), a flexible iterative data acquisition algorithm
that can be adapted to deep learning models over a wide range of tasks
through user-defined, goal-oriented objective functions. By explicitly
optimizing downstream objectives, GOIMDA enables active data
acquisition to substantially reduce the number of required data
points, leading to more efficient and cost-effective learning and
optimization.

Several directions remain for future work. First, GOIMDA currently
follows a fully sequential acquisition strategy, selecting a single
data point at each iteration and updating the model immediately. While
this approach maximizes data efficiency, it can increase overall
training time. A natural acceleration strategy is to acquire batches
of the top $b$ points with the lowest individual influence scores.
However, independent selection within a batch can reduce test
performance due to redundancy and correlation among the acquired
points, a limitation also observed in batch active learning
\citep{gal2017deep}. An important extension is therefore to account
for interactions among candidate points---such as mutual information---in
order to identify maximally influential acquisition batches while
preserving computational efficiency.

Second, GOIMDA relies on supervised learning models as surrogates,
using observed data to calibrate the current state of knowledge and to
estimate both candidate responses and their impact on the target
objective. Its effectiveness thus depends on the accuracy of these
predictions for unseen data. Given the abundance of unlabeled
information typically available in both candidate and target sets,
integrating semi-supervised learning techniques may better capture the
underlying structure of the feature space and further improve
predictive performance. We leave the development of such
semi-supervised extensions of GOIMDA to future work.

\vspace{20pt}

\textbf{Acknowledgements. } 
This work was supported in part by funding from the Office of Naval Research under grant N00014-23-1-2590, the National Science Foundation under grant No. 2310831, No. 2428059, No. 2435696, No. 2440954, a Michigan Institute for Data Science Propelling Original Data Science (PODS) grant, LG Management Development Institute AI Research, and Two Sigma Investments LP. Any opinions, findings, and conclusions or recommendations expressed in this material are those of the authors and do not necessarily reflect the views of the sponsors.
\clearpage

\bibliography{reference}

\clearpage
\appendix
 
\begin{center}
\textbf{\Large Supplementary Material}
\end{center}
\section{Deriving the Jacobian of minimizer w.r.t. parameters}\label{sec:appendix_partial_x}
By definition of $\hat{x}^\ast_\theta=:\hat{x}_{\min}(\theta)$, the first-order condition is
\begin{align}
0
&=\nabla_{x}\,\mathbb E_{y\sim p_\theta(\cdot\mid x)}[y\mid x]\Big\vert_{x=\hat{x}_{\min}(\theta)}
=\nabla_{x}\,A'\!\big(\eta_\theta(x)\big)\Big\vert_{x=\hat{x}_{\min}(\theta)}.
\label{eq:x_gradient_optimization}
\end{align}
Define
\[
g(\theta,x):=\nabla_x A'\!\big(\eta_\theta(x)\big)\in\mathbb R^{d_x}.
\]
Then $g(\theta,\hat x_{\min}(\theta))=0$. Fix $\hat\theta_t$ and perturb $\theta=\hat\theta_t+\delta$ with
\[
\Delta x:=\hat x_{\min}(\hat\theta_t+\delta)-\hat x_{\min}(\hat\theta_t).
\]
A first-order Taylor expansion gives
\begin{align*}
0
&=g(\hat\theta_t+\delta,\hat x_{\min}(\hat\theta_t+\delta))\\
&\approx g(\hat\theta_t,\hat x_{\min}(\hat\theta_t))
+\Big[\frac{\partial}{\partial\theta}g(\theta,x)\Big]_{\theta=\hat\theta_t,\;x=\hat x_{\min}(\hat\theta_t)}\,\delta
+\Big[\frac{\partial}{\partial x}g(\theta,x)\Big]_{\theta=\hat\theta_t,\;x=\hat x_{\min}(\hat\theta_t)}\,\Delta x .
\end{align*}
Since $g(\hat\theta_t,\hat x_{\min}(\hat\theta_t))=0$, we obtain
\begin{align*}
\Big[\frac{\partial}{\partial x}g(\theta,x)\Big]\Delta x
&=-
\Big[\frac{\partial}{\partial\theta}g(\theta,x)\Big]\delta,
\end{align*}
hence
\begin{align*}
\Delta x
&=
-\Big[\frac{\partial}{\partial x}g(\theta,x)\Big]^{-1}
\Big[\frac{\partial}{\partial\theta}g(\theta,x)\Big]\delta
\Big|_{\theta=\hat\theta_t,\;x=\hat x_{\min}(\hat\theta_t)}.
\end{align*}
Noting that $\frac{\partial}{\partial x}g(\theta,x)=\nabla_x^2 A'(\eta_\theta(x))$, this becomes
\begin{align*}
\hat{x}_{\min}(\hat\theta_t+\delta)-\hat{x}_{\min}(\hat\theta_t)
=
-\Big[\nabla_x^2 A'\!\big(\eta_\theta(x)\big)\Big]^{-1}
\Big[\frac{\partial}{\partial\theta}\nabla_x A'\!\big(\eta_\theta(x)\big)\Big]\delta
\Big|_{\theta=\hat\theta_t,\;x=\hat x_{\min}(\hat\theta_t)}.
\end{align*}
Sending $\delta\to 0$ yields the Jacobian (total derivative)
\begin{align*}
\frac{\partial \hat{x}_{\min}(\theta)}{\partial \theta}
=
-\Big[\nabla_x^2 A'\!\big(\eta_\theta(x)\big)\Big]^{-1}
\Big[\frac{\partial}{\partial\theta}\nabla_x A'\!\big(\eta_\theta(x)\big)\Big]
\Big|_{x=\hat x_{\min}(\theta)}.
\end{align*}

\section{Experiment datasets}
\subsection{Noisy black-box function} \label{appendix:noisy_black_box}
\textbf{2D Branin.} The first function is the Branin function $f(\mathbf{x}) = a\big[(x_2-bx_1^2+cx_1-r)^2+s(1-t)\cos(x_1)-q\big]$ where $a=1 / 51.95 $, $b=5.1/(4\pi)^2$, $c=5/\pi$, $r=6$, $s=10$, $t=1/(8\pi)$ and $q=44.81$, a rescaled form by \citep{Picheny2013opt_fun}. The function is evaluated on the square $x_1\in[0.,1.]$, $x_2\in[0.,1.]$. It has three global minima $f(\mathbf{x}^\ast)=-1.0474$.

\textbf{2D Drop-Wave function.} This two-dimensional, radially symmetric test function is highly multimodal due to its oscillatory cosine term. 
For $\mathbf x=(x_1,x_2)$ it is defined as
\[
f(\mathbf x)\;=\;-\frac{1+\cos\!\big(12\,\sqrt{x_1^2+x_2^2}\big)}{\,0.5\,(x_1^2+x_2^2)+2\,},
\]
and is typically evaluated on $x_i\in[-5.12,5.12]$. 
It attains the global minimum $f(\mathbf x^\ast)=-1$ at $\mathbf x^\ast=(0,0)$.

\textbf{5D Ackley function.} The Ackley function is generally expressed in the form $f(x)= -a\exp\big(-b\sqrt{\frac{1}{d}\sum_i x_i^2}\big) - \exp\big(\frac{1}{d}\sum_i \cos(cx_i)\big)+a+\exp(1)$, where $a = 20$, $b = 0.2$, $c = 2\pi$ and $\sigma = 0.5$. 
For the \textit{dimension} $d=2$, the function leads some optimization algorithms, particularly hill-climbing algorithms, to be trapped in one of its many local minima. The function is evaluated on the hypercube $x_i\in[-5,5]$, for $i=1,\ldots,d$. We set the dimension $d=5$ for higher dimensional analysis. The \textit{Global minimum}~$f(\mathbf{x}^\ast)=0$ is achieved at $\mathbf{x}^\ast = (0,\ldots,0)$.
\subsection{Hyperparameter tuning under distribution shift}\label{appendix:hyperparameter_tuning}

We use CIFAR-10 \citep{CIFAR10DATASET}, which contains $50,000$ training images and 10,000 test images across $10$ classes. The predictive model is a Pre-Activation Residual Network \citep{He2016IdentityMI}. We construct a target (unlabeled) evaluation set by selecting $3,000$ CIFAR-10 test images whose (unknown-to-the-learner) labels belong to a restricted class set $C:= \{C_1, C_2, C_3\}$ ($1,000$ images per class). We also construct a labeled source dataset $(X,Y)$ that is deliberately imbalanced: it contains $500$ labeled instances from each class in $C$, and $5,000$ labeled images from each of the remaining classes in $Y\setminus C$. Hyperparameter acquisition methods are evaluated by their impact on performance on the target subset, while training leverages the labeled source data.

\subsection{Predictive learning} \label{appendix:predictive_learning}
\subsubsection{MNIST active learning benchmark} \label{appendix:mnist}
MNIST \citep{mnist} contains $60,000$ training images and $10,000$ test images over 10 digit classes. Following \citep{gal2017deep}, we initialize the active learning loop with a random but class-balanced labeled set of 20 points (two per class). We hold out $1,024$ training points as a validation set; the remaining training points form the acquisition pool. Test accuracy is reported on the standard MNIST test set after each acquisition step.

\subsubsection{EMNIST Letters active learning}\label{appendix:emnist}
EMNIST Letters \citep{cohen2017emnist} is a balanced 26-class character recognition task with $124,800$ training images and $20,800$ test images. We set aside the last portion of the training set---equal in size to the test set---as a validation set, and use the remaining training points as the acquisition pool. We follow the same pool-based acquisition loop as in MNIST, starting from a small randomly initialized labeled set.

\subsubsection{Rotten Tomatoes movie-review phrases} \label{appendix:movie}
We evaluate binary sentiment classification using the Rotten Tomatoes phrase dataset originally collected by Pang and Lee \citep{movie2005}. The original labels are ordinal with five levels (``negative'', ``somewhat negative'', ``neutral'', ``somewhat positive'', ``positive''). To construct a binary task, we remove neutral phrases, map \{``negative'', ``somewhat negative''\} to $0$ and \{``somewhat positive'', ``positive''\} to $1$, and then randomly split the resulting dataset into $59,070$ training phrases, $1,024$ validation phrases, and $16,384$ test phrases. For text preprocessing, we use a bag-of-words representation; we discard words with total occurrence less than $10$, reducing the feature dimension from $15,186$ to $7,004$.

\end{document}